\documentclass[journal]{IEEEtran}
\usepackage{amsmath,amsfonts}
\usepackage{array}
\usepackage[caption=false,font=normalsize,labelfont=sf,textfont=sf]{subfig}
\usepackage{textcomp}
\usepackage{stfloats}
\usepackage{url}
\usepackage{verbatim}
\usepackage{graphicx}
\usepackage{cite}
\usepackage{xcolor}
\usepackage{colortbl}
\definecolor{baselinecolor}{gray}{.9}

\usepackage{threeparttable}
\usepackage{mathtools}
\usepackage{amsmath,amssymb,amsfonts}
\usepackage{amsthm}
\usepackage{bbm}
\usepackage{bm}
\usepackage{bbding} 
\usepackage{multirow}
\usepackage{booktabs}  
\usepackage{array}
\usepackage{amsfonts}
\usepackage{hyperref}
\usepackage{float}
\usepackage{cases}
\usepackage{graphicx}
\usepackage{capt-of}
\usepackage{varwidth}
\newtheorem{prop}{Proposition}
\newcommand\identity{1\kern-0.25em\text{l}}
\graphicspath{ {./figs/} }

\usepackage[utf8]{inputenc} % allow utf-8 input
\usepackage[T1]{fontenc}    % use 8-bit T1 fonts
\usepackage{nicefrac}       % compact symbols for 1/2, etc.
\usepackage{microtype}      % microtypography

\usepackage{algorithm}
\usepackage{algpseudocode}
\usepackage{bm}
\DeclareMathOperator{\Var}{Var}
\usepackage{makecell}

\begin{document}

\title{Efficient Training of Spiking Neural Networks by Spike-aware Data Pruning}

\author{
	Chenxiang~Ma, 	
    Xinyi~Chen,
    Yujie~Wu,
    Kay~Chen~Tan,~\IEEEmembership{Fellow,~IEEE},
	Jibin~Wu,~\IEEEmembership{Member,~IEEE}
 % <-this % stops a space
% \thanks{
% %This work was supported 
% }% <-this % stops a space
\thanks{\textit{Corresponding author: Jibin Wu (e-mail: jibin.wu@polyu.edu.hk).}

Chenxiang Ma and Xinyi Chen are with the Department of Data Science and Artificial Intelligence, The Hong Kong Polytechnic University, Hong Kong, SAR. 

Yujie Wu is with the Department of Computing, The Hong Kong Polytechnic University, Hong Kong, SAR. 

Kay Chen Tan is with the Department of Data Science and Artificial Intelligence and the Research Center of Data Science and Artificial Intelligence, The Hong Kong Polytechnic University, Hong Kong, SAR. 

Jibin Wu is with the Department of Data Science and Artificial Intelligence, the Department of Computing, and the Research Center of Data Science and Artificial Intelligence, The Hong Kong Polytechnic University, Hong Kong, SAR. }
}



\maketitle

\begin{abstract}
Spiking neural networks (SNNs), recognized as an energy-efficient alternative to traditional artificial neural networks~(ANNs), have advanced rapidly through the scaling of models and datasets. However, such scaling incurs considerable training overhead, posing challenges for researchers with limited computational resources and hindering the sustained development of SNNs. Data pruning is a promising strategy for accelerating training by retaining the most informative examples and discarding redundant ones, but it remains largely unexplored in SNNs. Directly applying ANN-based data pruning methods to SNNs fails to capture the intrinsic importance of examples and suffers from high gradient variance. To address these challenges, we propose a novel spike-aware data pruning (SADP) method. SADP reduces gradient variance by determining each example's selection probability to be proportional to its gradient norm, while avoiding the high cost of direct gradient computation through an efficient upper bound, termed spike-aware importance score. This score accounts for the influence of all-or-nothing spikes on the gradient norm and can be computed with negligible overhead. Extensive experiments across diverse datasets and architectures demonstrate that SADP consistently outperforms data pruning baselines and achieves training speedups close to the theoretical maxima at different pruning ratios. Notably, SADP reduces training time by 35\% on ImageNet while maintaining accuracy comparable to that of full-data training. This work, therefore, establishes a data-centric paradigm for efficient SNN training and paves the way for scaling SNNs to larger models and datasets. The source code will be released publicly after the review process.
\end{abstract}


\section{Introduction}
\IEEEPARstart{S}{piking} neural networks (SNNs) have gained considerable attention as a promising energy-efficient alternative to traditional artificial neural networks (ANNs)~\cite{roy2019towards,10636118,eshraghian2021training,2025cai,2025wangtoward,10691937}. Unlike ANNs, which rely on dense, continuous-valued activations that require frequent and computationally intensive updates across the entire network, SNNs encode and transmit information through sparse, binary spikes~\cite{maass1997networks}. This enables only a small subset of neurons to participate in processing afferent information while the majority remain quiescent at any given moment~\cite{10636118}. In addition, SNNs exhibit abundant spatiotemporal dynamics by temporally integrating incoming spikes into their membrane potentials and generating spikes only upon reaching threshold conditions~\cite{eshraghian2021training}. The spike-based representation and temporal integration enable inherently sparse, asynchronous, and event-driven computation. When deployed on neuromorphic hardware~\cite{davies2018loihi,pei2019towards,Darwin3,yao2024nc}, these characteristics translate into exceptional energy efficiency, thereby rendering SNNs highly advantageous in resource-constrained applications~\cite{spikingfullsubnet,9044638,hybridCoding,wu2020deep,2025quspiking}.

The rapid advancement of SNNs in recent years has been propelled by breakthroughs in scalable training methods~\cite{wu2018spatio,2023tandemwu,2022PTLwu} and the expansion of large-scale architectures~\cite{sewresnet,zhou2023spikformer,zhou2024qkformer,2025pami-transformer} and datasets~\cite{dvsgesture,li2017cifar10,hardvs,9311226}. In particular, the non-differentiable nature of the spike firing function posed a major obstacle to the development of scalable training methods~\cite{6469239,8305661,8351987,wu2018spiking}. This obstacle was largely overcome by surrogate gradients~\cite{neftci2019surrogate}, which enabled the application of gradient-based optimization, such as the back-propagation through time~(BPTT) algorithm~\cite{wu2018spatio,bptt}, to SNNs. Building upon this breakthrough, extensive efforts were devoted to enhancing the trainability of large-scale SNNs~\cite{2025DingAssisting}. In particular, advanced batch normalization methods~\cite{tdbn,duan2022temporal,Guo_2023_ICCV} were introduced to stabilize the optimization of deep architectures, adaptive surrogate gradients~\cite{wang2023adaptive,ijcai2023p335} were designed to facilitate gradient flow, and neuron models~\cite{tclif,plif,ALIF,dhsnn} incorporating more complex neuronal dynamics were developed to strengthen temporal processing capabilities~\cite{ma2025spiking}. Collectively, these advancements have enabled the expansion of SNNs to large-scale architectures with hundreds of layers and millions of parameters, exemplified by spiking variants of ResNets~\cite{sewresnet,2023spikingdeepresidualnetworks} and Transformers~\cite{zhou2023spikformer,zhou2024qkformer,2025pami-transformer}. Alongside this architectural growth, training datasets have also expanded considerably. Performance evaluation has shifted to large-scale datasets like ImageNet~\cite{deng2009imagenet}, which contains over one million training examples. Similarly, event-based datasets, collected by dynamic vision sensors (DVS)~\cite{4444573}, have grown from the one thousand examples in DvsGesture~\cite{dvsgesture} to more than one hundred thousand in the recent HAR-DVS dataset~\cite{hardvs}.

However, the increasing scale of both architectures and datasets leads to substantially longer training time, posing significant challenges in terms of extended development cycles and heavy computational demand. These high training costs are often unsustainable for researchers with limited high-performance computing resources. Consequently, enhancing the training efficiency of SNNs has become urgent for their scalable deployment and sustained advancement.

Preliminary attempts to improve the training efficiency of SNNs have primarily concentrated on learning algorithms~\cite{eprop,sltt,ottt} and neuron models~\cite{fang2023parallel,chen2024a}. From the algorithmic perspective, online learning rules~\cite{eprop,sltt,ottt} were developed to mitigate the high memory cost of BPTT. By decoupling the temporal dependencies of gradients, these approaches allow memory consumption to be independent of the total number of time steps. In addition, parallel spiking neuron models~\cite{fang2023parallel,chen2024a} were proposed to support simultaneous gradient computation along the temporal dimension, which accelerates training, particularly on long-sequence tasks. In contrast to these prior efforts, this work explores a fundamentally orthogonal perspective by targeting the dataset as the source of training efficiency gains. Given that large-scale datasets typically contain many redundant or uninformative examples, identifying a small and informative subset from the full training dataset, known as data pruning, provides a promising strategy for reducing training time while maintaining competitive performance.

In this article, we systematically investigate data pruning for SNNs. We start by directly applying existing data pruning methods designed for ANNs~\cite{forgetting, el2n, infobatch}. Our analysis reveals two critical reasons why these methods fall short when applied to SNNs. First, existing approaches struggle to efficiently and accurately evaluate data importance in the SNN training. Data pruning methods assign an importance score to each training example and then retain the most informative subset. Approaches such as Forgetting~\cite{forgetting} and EL2N~\cite{el2n} compute these scores by tracking training dynamics on the full dataset over tens of epochs, which is computationally expensive. InfoBatch~\cite{infobatch} offers a more efficient alternative by using the loss value as the importance score. However, this strategy is relatively limited for SNNs, where the gradient norm, a faithful measure of an example’s contribution to training, relies on sparse spike activity. Due to the all-or-nothing nature of spikes, a spike with a value of zero eliminates the corresponding gradient contribution, regardless of the loss value. As a result, the loss value correlates poorly with the gradient norm. This weak correlation becomes increasingly pronounced as spike sparsity increases (Figure~\ref{fig:ilustration}(a)), limiting the effectiveness of loss-based importance scores in SNNs.

\begin{figure*}[!t]
\centering\includegraphics[width=0.95\linewidth]{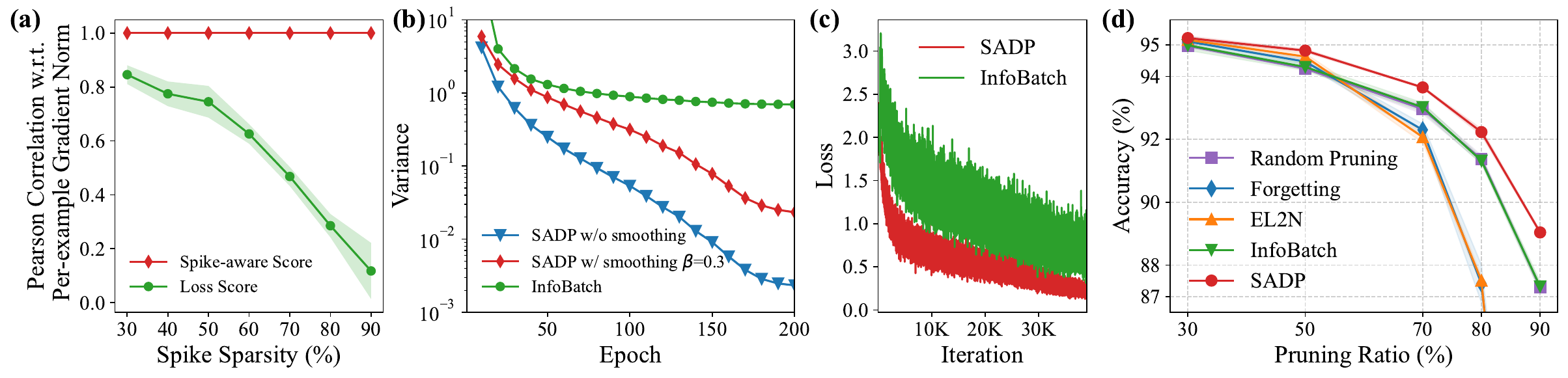}
\vspace{-2mm}
\caption{Comparison of SADP with existing data pruning methods.  
(a) Pearson correlation between the per-example gradient norm and different importance scores. Our spike-aware score maintains a significantly higher correlation than the loss score, particularly under high spike sparsity. 
(b) SADP reduces gradient variance throughout training.  
(c) This variance reduction leads to faster convergence.  
(d) SADP consistently achieves higher accuracy across different pruning ratios, with accuracy gains becoming more pronounced at higher ratios. Experiments are conducted with ResNet18 on CIFAR-10.
	}
\label{fig:ilustration}
\end{figure*}

Second, existing data pruning methods suffer from high gradient variance, hindering the convergence speed of SNN training. Forgetting~\cite{forgetting} and EL2N~\cite{el2n} select examples deterministically, leading to both biased and high-variance gradient estimates relative to using the full dataset. InfoBatch~\cite{infobatch} mitigates the bias issue by employing probabilistic sampling and reweighting the gradients of selected examples by the inverse of their probabilities. However, it incurs high gradient variance, which slows convergence and ultimately limits its performance (Figures~\ref{fig:ilustration}(b) and (c)).

To address these challenges, we propose spike-aware data pruning (SADP) (Figure~\ref{fig:pipeline}), a novel data pruning method tailored for SNNs. We formalize data pruning as a variance minimization problem and show that, under this objective, the optimal selection probability of a training example is approximately proportional to its gradient norm. Computing exact per-example norms, however, is prohibitively expensive. Therefore, we propose a spike-aware importance score that is an upper bound on the gradient norm, capturing the effect of all-or-nothing spikes while incurring negligible overhead. This score enables a tractable relaxation of the variance minimization objective, thereby closely approximating its optimum. To prevent training instability from low selection probabilities, we further propose a smoothing mechanism that enforces a minimum probability, stabilizing training while preserving low variance. Finally, to facilitate more efficient data usage over the course of training, we propose a dynamic pruning schedule that increases the pruning ratio over epochs while keeping the average ratio constant.

We conduct extensive experiments on both static and event-based vision datasets, including CIFAR-10~\cite{cifar}, CIFAR-100~\cite{cifar}, ImageNet~\cite{deng2009imagenet}, CIFAR10-DVS~\cite{li2017cifar10}, and HAR-DVS~\cite{hardvs}, across diverse architectures such as spiking VGG~\cite{simonyan2014very}, ResNet~\cite{he2016deep,sewresnet}, and Transformer~\cite{yao2024spikedriven}. Experimental results show that SADP consistently outperforms existing data pruning methods across different pruning ratios, with its advantage becoming more pronounced as the pruning ratio increases. Owing to its negligible computational overhead, SADP achieves the theoretical maximum reduction in training time, approximately equal to the given pruning ratio. For instance, SADP reduces training time by 70\% on CIFAR10-DVS~\cite{li2017cifar10} and 35\% on ImageNet~\cite{deng2009imagenet} without compromising accuracy. Furthermore, SADP demonstrates broad compatibility with various spiking models (such as PSN~\cite{fang2023parallel} and T-RevSNN~\cite{pmlr-v235-hu24q}), learning algorithms (including online~\cite{eprop,sltt} and local learning~\cite{decolle,ma2023ell}), and efficient inference techniques (such as quantization-aware training~\cite{wei2025qpsnn} and network pruning~\cite{li2024towards}). 

Our key contributions are summarized as follows:

\begin{itemize}
    \item We are the first to systematically investigate data pruning for SNNs and identify two challenges that limit the direct adoption of ANN-based methods, namely, ineffective data importance estimation and high gradient variance.
    \item We propose SADP, the first data pruning framework designed for SNNs. It incorporates a variance minimization formulation, a spike-aware importance score, a probability smoothing mechanism, and a dynamic pruning schedule to jointly overcome the identified challenges, achieving superior accuracy and training efficiency.
    \item Extensive experiments demonstrate that SADP outperforms prior data pruning methods, reduces training time by over 30\% with lossless accuracy, and generalizes well across diverse SNN methods.
\end{itemize}

The remainder of this paper is structured as follows. Section~\ref{sec:preliminaries} provides essential preliminaries on SNNs and data pruning. Section~\ref{sec:challenges} analyzes the challenges of data pruning in SNNs. Section~\ref{sec:SADP} details the proposed SADP method, and Section~\ref{sec:exp} presents extensive experimental evaluations and analyses. Finally, Section~\ref{sec:conclusion} concludes the paper.

\begin{figure*}[!t]
\centering\includegraphics[width=0.98\linewidth]{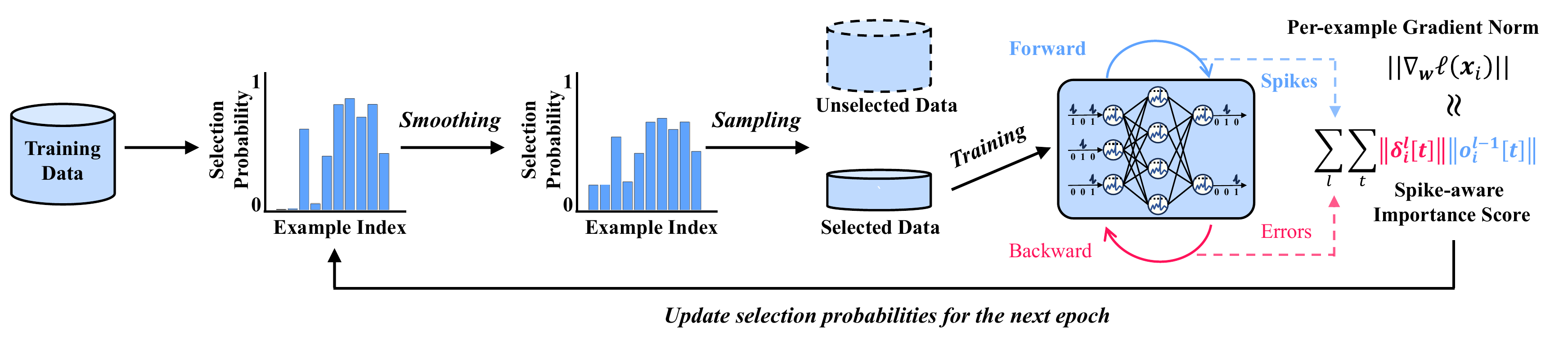}
\vspace{-2.5mm}
\caption{Illustration of SADP. At the start of each training epoch, selection probabilities are computed for all examples using the proposed spike-aware importance score, which provides an efficient and accurate approximation of the per-example gradient norm for variance minimization. To enhance training stability, the probabilities are smoothed to avoid extremely small values, after which a subset of examples is probabilistically sampled for training the SNN. Note that the spike-aware importance score explicitly captures the effect of sparse binary spikes on the per-example gradient norm, and its computation requires only quantities already available from forward and backward passes, thereby achieving both high efficiency and strong effectiveness.
	}
\label{fig:pipeline}
\end{figure*}

\section{preliminaries}
\label{sec:preliminaries}
\subsection{Spiking Neural Networks} SNNs are commonly built with Leaky Integrate-and-Fire~(LIF) neuron model~\cite{roy2019towards}. At time step~$t$, LIF neurons in layer $l$ integrate incoming spikes $\bm{o}^{l-1}[t]$ into their membrane potential $\bm{u}^{l}[t]$, which decays over time by a factor of $\lambda$. When $\bm{u}^{l}[t]$ surpasses a threshold~$\vartheta$, the neurons generate binary ($0$ or $1$) spikes $\bm{o}^{l}[t]$, followed by a reset process. This neuronal dynamics can be formulated as:
% \begin{align}
% \bm{u}^{l}[t] &= \lambda \bm{u}^{l}[t - 1]  + \bm{W}^{l} \bm{o}^{l-1}[t], \\
%     o_i^{l}[t]&=1,\ \text{if} \ u_i^{l}[t] >= \vartheta, \ \ 0, \ \text{otherwise},\\
% \bm{u}^{l}[t] &= \bm{u}^{l}[t] - \vartheta \bm{o}^{l}[t].
% \end{align}  
\begin{align}
\bm{u}^{l}[t] &= \lambda \bm{u}^{l}[t - 1]  + \bm{W}^{l} \bm{o}^{l-1}[t], \\
    o_j^{l}[t]&=\begin{cases}1, & u_j^{l}[t] >= \vartheta,\\0, & \text{otherwise},\end{cases}\\
\bm{u}^{l}[t] &= \bm{u}^{l}[t] - \vartheta \bm{o}^{l}[t].
\end{align}  
Large-scale SNNs are typically trained with BPTT~\cite{wu2018spatio}, which computes the gradient of a loss $\mathcal{L}$ with respect to the $l$-th layer weight~$\bm{W}^l$ as:
\begin{align}
\!\frac{\partial \mathcal{L}}{\partial \bm{W}^{l}} &= \sum_{t=1}^{T}\frac{\partial \mathcal{L}}{\partial \bm{u}^{l}[t]}  \frac{\partial \bm{u}^{l}[t]}{\partial \bm{W}^{l}} = \sum_{t=1}^{T}{\bm{\delta}^{l}[t]} {\bm{o}^{l-1}[t]}^{\!\top}, \label{eq:weight}
\\[1.5ex]
\bm{\delta}^{l}[t]\!&=\!\begin{cases} 
\frac{\partial \mathcal{L}}{\partial \bm{u}^L[T]},  &l\!=\!L~\text{and}~t\!=\!T, \\[1.5ex] 
\bm{\delta}^{L}[t\!+\!1] \!\frac{\partial \bm{u}^{L}[t+1]}{\partial \bm{u}^{L}[t]}+\frac{\partial \mathcal{L}}{\partial \bm{u}^L[t]},  &l\!=\!L~\text{and}~t\!<\!T, \\[1.5ex] 
\bm{\delta}^{l+1}[T] \!\frac{\partial \bm{u}^{l+1}[T]}{\partial \bm{o}^{l}[T]} \! \frac{\partial \bm{o}^{l}[T]}{\partial \bm{u}^{l}[T]},  &l\!<\!L~\text{and}~t\!=\!T, \\[1.5ex] 
\bm{\delta}^{l}\![t\!+\!1]\frac{\partial \bm{u}^{l}\![t+1]}{\partial \bm{u}^{l}\![t]} + \bm{\delta}^{l+1}\![t] \frac{\partial \bm{u}^{l+1}\![t]}{\partial \bm{u}^{l}\![t]},\!\!\!&\!\text{otherwise,}\! \end{cases}
\label{eq:grad_m}
\end{align}
where $\bm{\delta}^{l}[t]\coloneq\frac{\partial \mathcal{L}}{\partial \bm{u}^{l}[t]}$, represents the error back-propagated from the last layer and the last time step. $T$ denotes the total number of time steps. The gradient $\frac{\partial o_{j}^{l}[t]}{\partial u_{j}^{l}[t]}$ is zero except at the point where $u_{j}^{l}[t]\!=\!\vartheta$, where it becomes infinite. To address this issue, a continuous surrogate gradient function~\cite{neftci2019surrogate} is often adopted to replace the gradient during the backward pass.

\subsection{Data Pruning} Data Pruning seeks to select a small subset $\mathcal{S}$ of size $S$ from the full training dataset~$\mathcal{D}$ of size $N$, such that a model trained on $\mathcal{S}$ can achieve test performance comparable to that obtained when trained on~$\mathcal{D}$. The subset size $S$ is determined by a given pruning ratio $r$, defined as $r\!\coloneq\!1-\frac{S}{N}$. If the subset is efficiently identified, data pruning can reduce the training time by a factor of $r$. Existing data pruning methods assign an importance score to each training example and select those with higher scores to build the subset. For instance, Forgetting~\cite{forgetting} and EL2N~\cite{el2n} track training dynamics to estimate the forgetting and error L2-norm scores, respectively. Examples with scores above a threshold are selected before standard training begins and remain fixed throughout the training.

In contrast, the recently proposed InfoBatch~\cite{infobatch} assigns each example's loss as its importance score and adaptively selects a subset at the start of every epoch. To ensure unbiased gradient estimates between training on the selected subset and the full dataset, InfoBatch employs probabilistic sampling. Specifically, each example $i$ is selected with a predefined probability~$p_i\!=\!z$ if its score is less than the mean of all examples' scores, and with probability~$p_i\!=\!1$ otherwise. Note that the probability value $z$ is a hyperparameter that requires manual tuning. The gradients of the selected examples are subsequently upscaled by a factor of~$\frac{1}{p_i}$ to correct for the sampling bias.

\section{Challenges of Data Pruning in SNNs}
\label{sec:challenges}
This section details two critical challenges that limit the effectiveness of existing data pruning methods in SNNs: ineffective data importance estimation and high gradient variance.

\subsection{Ineffective Data Importance Estimation}
\label{subsec:imp_est}
Accurate estimation of data importance is crucial for data pruning, as it determines which examples are selected. However, existing methods fail to efficiently and effectively capture true data importance when applied to SNNs.  Forgetting~\cite{forgetting} and EL2N~\cite{el2n} estimate importance scores by iteratively tracking forgetting events and L2-norm of errors, respectively, which incur significant computational overhead. InfoBatch~\cite{infobatch} mitigates this cost by using the loss value as the importance score.
%requiring negligible computational cost. 
However, this approach is inadequate for SNNs due to the weak correlation between the loss and the gradient norm, as illustrated in Figure~\ref{fig:ilustration}(a) and further analyzed below.

We analyze the weight gradient norm of an example when training an SNN on a dataset $\mathcal{D}\!=\!\{\bm{x}_i\}_{i=1}^{N}$, since it directly reflects the example’s contribution to the training process. For clarity, we formulate the square of the gradient norm computed via the BPTT algorithm~\cite{wu2018spatio}, given by:
\begin{align}
\vspace{-0.4em}
\|\nabla_{\bm{W}}\ell\left(\bm{x}_i\right)\|^2\!&=\!\sum_{l\in \mathbb{L}}\|\nabla_{\bm{W}^l} \ell(\bm{x}_i)\|^2\!\\&=\!\sum_{l\in \mathbb{L}}\|\sum_{t=1}^{T}\frac{\partial\ell(\bm{x}_i)}{\partial \bm{u}^{l}[t]}\frac{\partial \bm{u}^{l}[t]}{\partial \bm{W}^{l}}\|^2\!\\&=\!\sum_{l\in \mathbb{L}}\|\sum_{t=1}^{T}{\bm{\delta}_i^{l}[t]}{\bm{o}_i^{l-1}[t]}^{\top}\|^2, \label{eq:weight_grad}
\vspace{-0.5em}
\end{align}
where $\mathbb{L}$ denotes the set of network layers and $\|\cdot\|$ represents the Euclidean norm of a vector or the Frobenius norm of a matrix. $\nabla_{\bm{W}^l} \ell(\bm{x}_i)
$ denotes the gradient of the example $i$'s loss $\ell(\bm{x}_i)$ with respect to the weights of the $l$-th layer~$\bm{W}^l$. $\bm{\delta}_i^{l}[t]\coloneq\frac{\partial \ell(\bm{x}_i)}{\partial \bm{u}^{l}[t]}$ is the backpropagated error corresponding to example $i$. It captures both spatial and temporal dependencies and is computed recursively as in Eq.~(\ref{eq:grad_m}).

As shown in Eq.~(\ref{eq:weight_grad}), the gradient norm of example~$i$ is governed by the outer products between its errors~$\{\bm{\delta}_i^l[t]\}_{l, t}$ and spikes~$\{\bm{o}_i^{l-1}[t]\}_{l,t}$. Due to the binary and sparse nature of spikes, any zero-valued $\bm{o}_i^{l-1}[t]$ eliminates the corresponding gradient contributions, regardless of the magnitude of the loss or errors. The all-or-nothing characteristic of spikes thus results in a weak correlation between the loss value~\cite{infobatch} and the weight gradient norm in SNNs, which becomes more pronounced as spike sparsity increases (Figure~\ref{fig:ilustration}(a)). Therefore, loss-based data pruning methods, such as InfoBatch~\cite{infobatch}, fail to identify truly important examples, leading to suboptimal performance in SNNs (Figure~\ref{fig:ilustration} (d)). 

Notably, computing the per-example gradient norm in Eq.~(\ref{eq:weight_grad}) is computationally expensive in popular SNN training frameworks such as PyTorch~\cite{pytorch}, SpikingJelly~\cite{spikingjelly}, and snnTorch~\cite{eshraghian2021training}. This is because these frameworks automatically average gradients over the batch dimension, rendering per-example gradients inaccessible, whereas computing gradients one example at a time is prohibitively inefficient.

As a result, efficiently estimating data importance in a way that closely approximates the per-example gradient norm remains a critical challenge for data pruning in SNNs.

\subsection{High Gradient Variance and Slow Convergence Speed}
\label{subsec:high_var}
While InfoBatch~\cite{infobatch} ensures unbiased gradient estimates between the selected subset and the full dataset, it suffers from high gradient variance, which slows SNN training convergence. This issue is visualized in Figures~\ref{fig:ilustration}(b) and (c) and further analyzed in the following.

Here, we consider the general weight update rule for training an SNN at iteration~$m\!+\!1$, given by:
\begin{equation}
\bm{W}_{m+1} = \bm{W}_{m} - \eta \nabla_{\bm{W}}\mathcal{L},
\label{eq:sgd_weight}
\end{equation}
where $\eta$ denotes the learning rate, and $\nabla_{\bm{W}} \mathcal{L}$ is the gradient of the total loss $\mathcal{L}$ with respect to the weights~$\bm{W}$. The total loss $\mathcal{L}$ is defined as the average loss over all examples, i.e. $\mathcal{L}\coloneq\frac{1}{N} \sum_{i=1}^{N}\ell\left(\bm{x}_i\right)$.

Then, we formalize the convergence speed $\mathcal{C}$ as the expected reduction in distance to the optimal weights~$\bm{W}^*$ between two consecutive iterations. Specifically, using the weight update rule in Eq.~(\ref{eq:sgd_weight}), the convergence speed is defined as:
\begin{align}
\mathcal{C} &\coloneq -\mathbb{E} \Big[ \| \bm{W}_{m+1} - \bm{W}^* \|^2 - \| \bm{W}_m - \bm{W}^* \|^2 \Big] \nonumber\\
&= -\mathbb{E} \Big[ (\bm{W}_m - \eta \nabla_{\bm{W}}\mathcal{L})^\top (\bm{W}_m - \eta \nabla_{\bm{W}}\mathcal{L}) \nonumber\\
&\quad\quad + 2 \eta {\nabla_{\bm{W}}\mathcal{L}}^\top\bm{W}^* - \bm{W}_m^\top \bm{W}_m \Big] \nonumber\\
&= -\mathbb{E} \Big[ -2\eta (\bm{W}_m - \bm{W}^*)^\top \nabla_{\bm{W}}\mathcal{L} 
   + \eta^2 {\nabla_{\bm{W}}\mathcal{L}}^\top \nabla_{\bm{W}}\mathcal{L} \Big] \nonumber\\
&= 2\eta (\bm{W}_m - \bm{W}^*)^\top \mathbb{E}[ \nabla_{\bm{W}}\mathcal{L}]
   - \eta^2 \mathbb{E}[ \nabla_{\bm{W}}\mathcal{L} ]^\top \mathbb{E}[ \nabla_{\bm{W}}\mathcal{L}] \nonumber\\
&\quad - \eta^2  \Var[\nabla_{\bm{W}}\mathcal{L}].
\label{eq:convergence_vs_variance}
\end{align}
Eq.~(\ref{eq:convergence_vs_variance}) reveals that the variance term $\Var[\nabla_{\bm{W}}\mathcal{L}]$ introduces noise into the weight update and reduces convergence speed. %High gradient variance can significantly hinder training efficiency. 

This insight underscores an important objective for data pruning: minimizing gradient variance to accelerate convergence and improve training efficiency.

\section{Spike-Aware Data Pruning (SADP)}
\label{sec:SADP}
In this section, we propose SADP, which jointly addresses the aforementioned two challenges of accurately and efficiently estimating data importance and minimizing gradient variance. SADP is grounded in a variance minimization framework, where each example's optimal selection probability is proportional to its gradient norm (Section~\ref{sec:low-var}). To approximate this quantity with negligible computational overhead, we propose an efficient upper bound on the gradient norm, referred to as spike-aware importance score (Section~\ref{sec:spk_proxy}). To enhance stability during training, SADP further incorporates a probability smoothing mechanism that prevents selection probabilities from collapsing to very small values (Section~\ref{sec:smoothing}). Finally, we propose a dynamic pruning ratio schedule that promotes more efficient and balanced utilization of data over the course of training~(Section~\ref{sec:schedule}). An overview of SADP is presented in Figure~\ref{fig:pipeline}, with pseudocode provided in Algorithm~\ref{alg:sadp}.

\begin{algorithm}[!t]
   \caption{Spike-Aware Data Pruning (SADP)}
   \label{alg:sadp}
\begin{algorithmic}[1]  % [1] for line numbering
   \State \textbf{Input:} SNN model~$f_{\bm{W}}$, dataset~$\mathcal{D}\!=\!{\{\bm{x}_i\}}_{i=1}^N$, training epochs $K$, pruning ratio $r$, maximum ratio $r_{\mathrm{max}}$, smoothing constant $\beta$, layer index $l$, time window $T$, subset size $S$, batch size $B$.
   \State Initialize training examples' importance scores $\{G_i\}_{i=1}^N$.
   \For {epoch $k = 1$ to $K$}
       \State Compute pruning ratio $r_k$ using Eq.~(\ref{eq:compute_rk}).
       \State Compute examples' selection probabilities $\mathbf{p}\!=\!(p_1, \ldots, p_n)$ using Eqs.~(\ref{eq:compute_offset}) and (\ref{eq:compute_smoothed_pi}). 
       \State Build subset $\mathcal{S} = \{ \bm{x}_i \in \mathcal{D} \mid m_i = 1, m_i \sim \text{Bernoulli}(p_i) \}$.
       \For{iteration $m = 1$ to $\frac{S}{B}$}
           \State Sample batch $\{\bm{x}_i\}_{i=1}^B \subseteq \mathcal{S}$.
           \State Perform forward pass and record $\{ \bm{o}_i^{l-1}[t]\}_{i, l, t}$.
           \State Compute the scaled loss based on Eq.~(\ref{eq:est_w_grad}).
           \State Perform backward pass and record $\{\bm{\delta}_i^l[t]\}_{i, l, t}$.
           \State Compute and update $\{G_i\}_{i=1}^B$ using Eq.~(\ref{eq:compute_g_i}).
           \State Update model parameters of $f_{\bm{W}}$.
       \EndFor
   \EndFor
\end{algorithmic}
\end{algorithm}

\subsection{Low-variance Data Pruning}
\label{sec:low-var}
We begin by formulating the estimated gradient computed from a pruned subset of size $S$. At the start of each epoch, every example $i$ is independently selected according to a Bernoulli distribution with a selection probability $p_i$. Let $m_i\in\{0, 1\}$ denote a binary random variable indicating whether example $i$ is selected. The estimated gradient over the dataset~$\mathcal{D}\!=\!\{\bm{x}_i\}_{i=1}^{N}$ can then be given as:
\begin{equation}
\label{eq:est_w_grad}
\hat{\nabla}_{\bm{W}}\mathcal{L}=\frac{1}{S}\cdot\frac{S}{N}\sum_{i=1}^N \frac{\nabla_{\bm{W}}\ell\left(\bm{x}_i\right)\cdot m_i}{p_i},
\end{equation}
where each selected example’s gradient is reweighted by the inverse of its selection probability $p_i$, ensuring that the estimated gradient remains unbiased with respect to the full-data gradient~\cite{infobatch}.

Next, we elucidate how to reduce the variance of the estimated gradient $\Var[\hat{\nabla}_{\bm{W}}\mathcal{L}]$, which can be formulated as the following optimization problem: 
%This objective is formulated as the following optimization problem:
\begin{align}
\min_{\textbf{p}}\quad&\Var[\hat{\nabla}_{\bm{W}}\mathcal{L}]=\frac{1}{N^2}\sum_{i=1}^N \frac{(1-p_i)\|\nabla_{\bm{W}}\ell\left(\bm{x}_i\right)\|^2}{p_i} \nonumber\\
\textrm{s.t.} \quad & \sum_{i=1}^N p_i = S \quad\text{and}\quad 0\leq p_i\leq1,   
\label{eq:min_var}
\end{align}
where the gradient variance is governed by the selection probabilities $\mathbf{p}\!=\!(p_1, \ldots, p_N)$. The constraint $\sum_{i=1}^N p_i\!=\!S$ ensures that the expected number of selected examples equals the given subset size $S$. Proposition~\ref{prop:var_min} derives the optimal form of $\mathbf{p}$ that minimizes the gradient variance:
\begin{prop}
\label{prop:var_min}
For the Bernoulli sampling gradient estimator in Eq.~(\ref{eq:est_w_grad}), the optimal selection probabilities $\mathbf{p}^*\!=\!(p_1^*, \ldots, p_N^*)$ that minimize the gradient variance $\Var[\hat{\nabla}_{\bm{W}}\mathcal{L}]$ are
\begin{equation} p_i^\ast = \frac{\min ( \|\nabla_{\bm{W}} \ell(\bm{x}_i)\|, \alpha)\cdot S}{\sum_{j=1}^{N} \min ( \|\nabla_{\bm{W}} \ell(\bm{x}_j)\|, \alpha )}, \quad i = 1, \ldots, N, 
\label{eq:opt_prob}\end{equation} where the clipping threshold $\alpha$ is defined as 
\begin{equation}
\alpha = \frac{\sum_{i=1}^{N - M} \|\nabla_{\bm{W}} \ell(\bm{x})\|_{(i)}}{S - M},   
\end{equation}
and $\|\nabla_{\bm{W}}\ell(\bm{x})\|_{(1)}\!\leq\!\cdots\!\leq\!\|\nabla_{\bm{W}} \ell(\bm{x})\|_{(N)}$ denote the sorted gradient norms. The integer $M$ satisfies:
\begin{align}
\frac{\|\nabla_{\bm{W}} \ell(\bm{x})\|_{(N - M)}}{\sum_{i=1}^{N - M} \|\nabla_{\bm{W}} \ell(\bm{x})\|_{(i)}} &< \frac{1}{S - M} \quad and \\ 
\frac{\|\nabla_{\bm{W}} \ell(\bm{x})\|_{(N - M + 1)}}{\sum_{i=1}^{N - M + 1} \|\nabla_{\bm{W}} \ell(\bm{x})\|_{(i)}} &\geq \frac{1}{S - M + 1}.    
\end{align}.
\end{prop} 
\begin{proof}
Let $g_i\coloneq \|\nabla_{\bm{W}} \ell(\bm{x}_i)\|$, and the ordered set $\{g_i\}_{i=1}^N$ is represented by $g_{(1)} \leq g_{(2)} \leq \dots \leq g_{(N)}$. 
% The variance minimization problem can be formulated as:
% \begin{equation}
% \begin{aligned}
% \min_{\mathbf{p}}\quad &\sum_{i=1}^N \frac{(1-p_i)g_{(i)}^2}{p_i} \\
% \text{s.t.} \quad & \sum_{i=1}^N p_i = S \quad\text{and}\quad 0\leq p_i\leq1.    \\
% \end{aligned}
% \end{equation}
The Lagrangian for this constrained optimization problem is
\begin{align}
E(\mathbf{p}, \lambda, \bm{\mu}, \bm{\nu})&=\sum_{i=1}^N \frac{(1-p_i)g_{(i)}^2}{p_i}\\&+\lambda(\sum_{i=1}^Np_i-S)+\sum_{i=1}^N \mu_i\left(p_i-1+\nu_i^2\right),
\end{align}
where $\lambda$, $\bm{\mu}=(\mu_1, \ldots, \mu_N)$, and $\bm{\nu}=(\nu_1, \ldots, \nu_N)$ are the multipliers. By taking partial derivatives of the Lagrangian with respect to $\mathbf{p}$ and the multipliers, we obtain the Karush-Kuhn-Tucker~(KKT) conditions~\cite{boyd2004convex}:
\begin{align}
\frac{\partial E}{\partial p_i} &= -\frac{g_{(i)}^2}{p_i^2} + \lambda + \mu_i = 0, \quad i=1, \dots, N, \label{aeq:pi} \\
\frac{\partial E}{\partial \lambda} &= \sum_{i=1}^N p_i - S = 0, \label{aeq:lambda} \\
\frac{\partial E}{\partial \mu_i} &= p_i - 1 + \nu_i^2 = 0, \quad i=1, \dots, N, \label{aeq:mui} \\
\frac{\partial E}{\partial \nu_i} &= 2\mu_i \nu_i = 0, \quad i=1, \dots, N, \label{aeq:nui} \\
\mu_i &\geq 0, \quad i=1, \dots, N. \label{aeq:mui_pos}
\end{align}
Eq.~(\ref{aeq:nui}) implies that at least one of $\mu_i$ and $\nu_i$ must be zero. If $\mu_i=0$, then from Eq.~(\ref{aeq:pi}), we have $p_i=\frac{g_{(i)}}{\sqrt{\lambda}}$. Since $p_i\leq1$, we have $g_{(i)}\leq\sqrt{\lambda}$. If $\nu_i=0$, then from Eq.~(\ref{aeq:mui}), we have $p_i=1$ and $g_{(i)}\geq\sqrt{\lambda}$. Let $M$ be an integer such that $g_{(N-M)}<\sqrt{\lambda}$ and $g_{(N-M+1)} \geq \sqrt{\lambda}$. Then, from Eq.~(\ref{aeq:lambda}), we have: 
\begin{equation}
    \sqrt{\lambda} = \frac{\sum_{i=1}^{N-M}g_{(i)}}{S-M} \coloneq \alpha.
\end{equation}
Next, we can get:  
\begin{equation}
\sum_{i=1}^N\min ( g_{(i)}, \alpha)=\sum_{i=1}^{N-M}g_{(i)} + M\cdot\alpha = S\cdot\alpha.
\end{equation}
Therefore, for $i=1,\ldots,N-M$, we have:
\begin{equation}
p_i^\ast=\frac{g_{(i)}}{\alpha} = \frac{\min ( g_{(i)}, \alpha)\cdot S}{\sum_{j=1}^{N} \min ( g_{(j)}, \alpha )}.
\end{equation}
For $i=N-M+1,\ldots,N$, we have:
\begin{equation}
p_i^\ast=1=\frac{\alpha}{\alpha} = \frac{\min ( g_{(i)}, \alpha)\cdot S}{\sum_{j=1}^{N} \min ( g_{(j)}, \alpha )}.
\end{equation}
Thus, we derive the optimal probabilities. The proof is completed. 
\end{proof}
Proposition~\ref{prop:var_min} suggests that the optimal selection probability for each example is proportional to its gradient norm, and larger gradient magnitudes lead to higher probabilities. The probabilities are upper bounded by $\alpha$ to ensure that they remain within the valid range.

However, computing $\alpha$ requires sorting the full set of gradient norms, which incurs a time complexity of $\mathcal{O}(N \log N)$ and poses a bottleneck for large-scale datasets. To address this limitation, we propose a computationally efficient alternative that avoids sorting while maintaining the same optimal solution. 
Specifically, we iteratively identify and remove examples whose selection probabilities reach $1$, and redistribute the probabilities of the remaining examples.
Formally, let $\mathcal{R}\!=\!\left\{ \bm{x}_i\!\in\!\mathcal{D}\!\mid 0\!\leq\!p_i\!<\!1\!\right\}$ denote the set of examples with selection probabilities less than $1$, and let $\mathcal{R}_{\mathrm{nz}}\!\subseteq\!\mathcal{R}$ be the subset with \textit{non-zero} probabilities. We denote their sizes by $|\mathcal{R}|$ and $|\mathcal{R}_{\mathrm{nz}}|$, respectively. We initialize $\mathcal{R}\!=\!\mathcal{R}_{\mathrm{nz}}\!=\!\mathcal{D}$ and iteratively update the selection probabilities using:
\begin{equation} 
p_i^\ast = \frac{ \|\nabla_{\bm{W}} \ell(\bm{x}_i) \| \cdot (S-N+|\mathcal{R}|)}{\sum_{j\in\mathcal{R}_{\mathrm{nz}}} \|\nabla_{\bm{W}} \ell(\bm{x}_j)\|}, \quad \forall \ i\in{\mathcal{R}}_{\mathrm{nz}}.
\label{eq:eff_iter}
\end{equation}
Any probability value not less than $1$ is bounded at $1$, and the corresponding example is removed from the set $\mathcal{R}$. This process is repeated iteratively until all probabilities lie within the valid range~$[0, 1]$.
Although each iteration of Eq.~(\ref{eq:eff_iter}) has a computational complexity of $\mathcal{O}(|\mathcal{R}_{\mathrm{nz}}|)$, the required iteration number in practice is very small, resulting in lower overhead compared to the sorting-based approach.

\subsection{Spike-aware Importance Score}
\label{sec:spk_proxy}
According to Eq.~(\ref{eq:eff_iter}), the optimal selection probabilities that minimize gradient variance are proportional to the per-example gradient norm, aligning with the intuition that the gradient norm reflects each example’s true contribution to the training process. However, as analyzed in Section~\ref{subsec:imp_est}, performing a backward pass for each example to compute its gradient norm is prohibitively inefficient. 
%The common workaround that uses the per-example loss~\cite{infobatch} is inadequate due to the weak correlation between the loss and the gradient norm in SNNs (see Figure~\ref{fig:ilustration}(a)).
A common alternative is to use the per-example loss as a proxy~\cite{infobatch}, but this approach proves inadequate in SNNs due to the weak correlation between the loss and the gradient norm.

To address this challenge, we notice that the per-example errors  $\{\bm{\delta}_i^{l}[t]\}_{i,l,t}$ and the corresponding spikes $\{\bm{o}_i^{l- 1}[t]\}_{i, l, t}$ have been readily available after the backward pass of BPTT~\cite{wu2018spatio}. The per-example gradient norm can thus be reconstructed by computing the outer products between these errors and spikes, as described in Eq.~(\ref{eq:weight_grad}). However, calculating the outer products still incurs considerable computational cost.
 
To further eliminate the outer product operations, we observe that the objective of the variance minimization problem in Eq.~(\ref{eq:min_var}) is a function of the per-example gradient norm. This allows us to relax the minimization objective by introducing an efficient upper bound $G_i$ on the gradient norm, $G_i \geq \|\nabla_{\bm{W}} \ell(\bm{x}_i)\|$. Under this relaxation, the optimal selection probabilities are guaranteed to be similar to those derived from the original objective. To this end, we propose an outer-product-free upper bound, referred to as spike-aware importance score, which is formally defined in Proposition~\ref{prop:spikeAwareProxy}.
\begin{prop}
\label{prop:spikeAwareProxy}
The variance minimization objective in Eq.~(\ref{eq:min_var}) can be relaxed by replacing the per-example gradient norm \( \| \nabla_{\bm{W}} \ell(\bm{x}_i) \| \) with an upper bound \( G_i \), defined as:
\begin{equation} G_i = \sum_{l\in \mathbb{L}}\sum_{t=1}^T \| \bm{\delta}_i^{l}[t] \| \cdot \| \bm{o}_i^{l- 1}[t]\|, 
\label{eq:compute_g_i}
\end{equation} 
where $\bm{\delta}_i^{l}[t]$ and $\bm{o}_i^{l-1}[t]$ represent example $i$'s errors and spikes at time $t$ and layer $l$, respectively. 
The optimal selection probability for each example becomes $\hat{p}_i^{\ast} = \frac{G_i \cdot (S-N+|\mathcal{R}|)}{\sum_{j \in \mathcal{R}_{\mathrm{nz}}} G_j}$, $\forall \ i\in{\mathcal{R}}_{\mathrm{nz}}$.
\end{prop}

\begin{proof}
We start by relaxing the original objective in Eq.~(\ref{eq:min_var}), which aims to minimize the variance of the gradient estimator, by introducing an upper bound $G_i \geq \|\nabla_{\bm{W}} \ell(\bm{x}_i)\|$ for each training example $i$. This yields the following relaxed objective:
\begin{equation}
\min_{\mathbf{p}}\quad\sum_{i=1}^N \frac{(1-p_i)\|\nabla_{\bm{W}}\ell\left(\bm{x}_i\right)\|^2}{p_i} \leq \min_{\mathbf{p}}\quad\sum_{i=1}^N \frac{(1-p_i) G_i^2}{p_i}.
\label{app:eq:relaxed_prob}
\end{equation}
To make this surrogate practically useful, we now derive a form of $G_i$ that is significantly cheaper to compute than the exact gradient norm. We consider two common cases: fully-connected and convolutional layers.

For a fully-connected layer $l$, the gradient with respect to the weights can be expressed as:
\begin{equation}
\nabla_{\bm{W}^l} \ell(\bm{x}_i) = \sum_{t=1}^T \bm{\delta}_i^l[t] \cdot \bm{o}_i^{l-1}[t]^\top.
\end{equation}
% where $\bm{\delta}_i^l[t] = \frac{\partial \ell(\bm{x}_i)}{\partial \bm{u}^l[t]}$ is the error signal at layer $l$ and time $t$, and $\bm{o}_i^{l-1}[t]$ is the spike output from the previous layer. 
By the triangle inequality and sub-multiplicativity of the norm, we obtain:
\begin{align}
\|\nabla_{\bm{W}^l} \ell(\bm{x}_i)\| 
&= \left\| \sum_{t=1}^T \bm{\delta}_i^l[t] \cdot \bm{o}_i^{l-1}[t]^\top \right\| \notag \\
&\leq \sum_{t=1}^T \| \bm{\delta}_i^l[t] \cdot \bm{o}_i^{l-1}[t]^\top \| \notag \\
&\leq \sum_{t=1}^T \| \bm{\delta}_i^l[t] \| \cdot \| \bm{o}_i^{l-1}[t] \|.
\end{align}

For convolutional layers, the weight sharing mechanism changes the structure of the gradient. Each weight kernel slides over spatial regions of the input, and input elements are reused across multiple patches. The gradient with respect to the convolutional kernel is given by:
\begin{equation}
\nabla_{\bm{W}^l} \ell(\bm{x}_i) = \sum_{t=1}^T \bm{\delta}_i^l[t] \cdot \mathrm{unfold}(\bm{o}_i^{l-1}[t])^\top,
\end{equation}
where $\mathrm{unfold}(\cdot)$ extracts sliding local patches from the spike tensor and flattens them into columns. Applying norm inequalities again yields:
\begin{align}
\|\nabla_{\bm{W}^l} \ell(\bm{x}_i)\|
&\leq \sum_{t=1}^T \| \bm{\delta}_i^l[t] \| \cdot \| \mathrm{unfold}(\bm{o}_i^{l-1}[t]) \|.
\end{align}
We now bound the unfolded input norm. Since unfolding duplicates entries across overlapping patches, the norm of the unfolded input can be upper bounded by:
\begin{equation}
\| \mathrm{unfold}(\bm{o}_i^{l-1}[t]) \| \leq \sqrt{N_{\mathrm{patch}}^l} \cdot \| \bm{o}_i^{l-1}[t] \|,
\end{equation}
where $N_{\mathrm{patch}}^l$ denotes the number of patches in layer $l$. By substituting, we obtain:
\begin{align}
\|\nabla_{\bm{W}^l} \ell(\bm{x}_i)\| 
&\leq \sqrt{N_{\mathrm{patch}}^l} \cdot \left(\sum_{t=1}^T \| \bm{\delta}_i^l[t] \| \cdot \| \bm{o}_i^{l-1}[t] \|\right).
\end{align}
Let $\mathbb{L}$ denote the set of trainable layers. Then the total gradient norm satisfies:
\begin{align}
\| \nabla_{\bm{W}} \ell(\bm{x}_i) \| 
&= \sqrt{ \sum_{l \in \mathbb{L}} \| \nabla_{\bm{W}^l} \ell(\bm{x}_i) \|^2 } \notag \\
&\leq \sum_{l \in \mathbb{L}} \| \nabla_{\bm{W}^l} \ell(\bm{x}_i) \| \notag \\
&\leq N_{\mathrm{patch}}^\ast \cdot \left(\sum_{l \in \mathbb{L}} \sum_{t=1}^T \| \bm{\delta}_i^l[t] \| \cdot \| \bm{o}_i^{l-1}[t] \|\right),
\end{align}
where $N_{\mathrm{patch}}^\ast\!=\!\max_{l \in \mathbb{L}} \sqrt{N_{\mathrm{patch}}^l}$ is a layer-independent constant that can be ignored when computing normalized probabilities.

We thus obtain the following upper bound:
\begin{equation}
G_i = \sum_{l \in \mathbb{L}} \sum_{t=1}^T \| \bm{\delta}_i^l[t] \| \cdot \| \bm{o}_i^{l-1}[t] \|,
\end{equation}
which uses quantities already available during BPTT and is cheap to compute. The optimal per-example selection probability thus becomes:
\begin{equation}
\hat{p}_i^\ast = \frac{G_i \cdot (S - N + |\mathcal{R}|)}{\sum_{j \in \mathcal{R}_{\mathrm{nz}}} G_j}, \quad \forall \ i\in{\mathcal{R}}_{\mathrm{nz}}.
\end{equation}
The proof is completed.
\end{proof}
The spike-aware importance score eliminates the need for outer products  and captures the joint influence of binary spikes and spatio-temporal dependent errors, resulting in a strong correlation with the per-example gradient norm. In addition, both components required to compute the score are readily available after the backward pass of BPTT, making the method highly efficient and well-suited for large-scale SNN training.

\subsection{Smoothing Selection Probabilities} 
\label{sec:smoothing}
Another issue is that examples with very low selection probabilities can produce training instability when selected. This instability arises because, to ensure unbiased gradient estimation, the loss of each selected example is scaled by the inverse of its selection probability, as shown in Eq.~(\ref{eq:est_w_grad}). 
SADP computes selection probabilities in proportion to spike-aware importance scores, and consequently, examples with low scores are assigned very small probabilities. When they are occasionally sampled, their scaled gradients become disproportionately large, potentially destabilizing the training process. 

To address this issue, we introduce a smoothing mechanism that enforces a minimum probability. Specifically, if the smallest probability falls below a predefined constant $\beta$, we add an offset~$\gamma$ to every spike-aware score $G_i$ such that the smallest probability equals $\beta$. The offset $\gamma$ is calculated by:
\begin{equation} 
\frac{(G_{\mathrm{min}} + \gamma) \cdot (S - N + |\mathcal{R}|)}{\sum_{j \in \mathcal{R}_{\mathrm{nz}}} (G_j + \gamma)} = \beta, 
\label{eq:compute_offset}
\end{equation}
where $G_{\mathrm{min}}$ is the smallest non-zero score among all examples in $\mathcal{R}$. The smoothed probability is then given by:
\begin{equation} 
p_i = \frac{(G_i + \gamma) \cdot (S - N + |\mathcal{R}|)}{\sum_{j \in \mathcal{R}_{\mathrm{nz}}} (G_j + \gamma)}, \quad \forall \ i\in{\mathcal{R}}_{\mathrm{nz}}.
\label{eq:compute_smoothed_pi}
\end{equation} 
This smoothing strategy prevents excessive gradient scaling from low-probability examples while maintaining low variance, thereby contributing to more stable and efficient SNN training.

\subsection{Dynamic Pruning Schedule}
\label{sec:schedule}
As training progresses, a growing number of examples become well-learned and contribute less to gradient updates. Motivated by this intuition, we propose a dynamic pruning schedule in which the pruning ratio increases over the course of training: fewer examples are pruned in the early epochs to allow the model to learn broadly from diverse data, while more aggressive pruning is applied in later epochs when less informative examples dominate. The average pruning ratio over epochs keeps unchanged.

Specifically, let $K$ denote the number of training epochs. $r$ and $r_{\mathrm{max}}$ represent the average and maximum pruning ratios, respectively. The pruning ratio $r_k$ at the $k$-th epoch increase linearly from $2r-r_{\mathrm{max}}$ to $r_{\mathrm{max}}$ over the course of training:
\begin{equation}
r_k = 2r - r_{\mathrm{max}} + \frac{k(2r_{\mathrm{max}}-2r)}{K}, \quad k = 1, \ldots, K.
\label{eq:compute_rk}
\end{equation}

\section{Experiments}
\label{sec:exp}
This section presents a comprehensive evaluation of SADP across a wide range of datasets and architectures. In Section~\ref{subsec:per_evl}, we compare SADP with data pruning baselines in terms of accuracy and training efficiency. Section~\ref{subsec:ablation} investigates the contribution of individual components through detailed ablation studies. Section~\ref{subsec:grad_analysis} further analyzes gradient approximation and variance reduction, and evaluates the proposed smoothing mechanism for addressing gradient scaling. Finally, Section~\ref{subsec:broad_app} demonstrates the broad applicability and compatibility of SADP across diverse SNN methods. 

\begin{table*}[!t]
\caption{Training configurations and hyperparameter settings.}
\label{tab:train_details}
\centering
\begin{tabular}{cccccccccc}
\hline\hline
Dataset &
  \# Epochs &
  Optimizer &
  \begin{tabular}[c]{@{}c@{}}Learning \\ Rate\end{tabular} &
  \begin{tabular}[c]{@{}c@{}}Learning Rate \\ Schedule\end{tabular} &
  \begin{tabular}[c]{@{}c@{}}Batch \\ Size\end{tabular} &
  \begin{tabular}[c]{@{}c@{}}Weight \\ Decay\end{tabular} &
  \begin{tabular}[c]{@{}c@{}}Neuronal\\  Decay $\lambda$\end{tabular} &
  Threshold $\vartheta$&
  \begin{tabular}[c]{@{}c@{}}\# Time \\ Steps (T)\end{tabular} \\ \hline
CIFAR-10            & 200 & SGD   & 0.2  & Cosine Annealing & 128 & $0.00005$ & 0.1 & 1.0 & 4  \\
CIFAR-100           & 200 & SGD   & 0.2  & Cosine Annealing & 128 & $0.00005$ & 0.1 & 1.0 & 4  \\
ImageNet (Pretrain) & 100 & SGD   & 0.25 & Cosine Annealing & 512 & $0.00001$ & 0.2 & 1.0 & 1  \\
ImageNet (Finetune) & 10  & SGD   & $0.001$ & Cosine Annealing & 128 & 0    & 0.2 & 1.0 & 4  \\
CIFAR10-DVS        & 100 & AdamW & $0.001$ & Cosine annealing & 100 & $0.0005$ & 0.1 & 1.0 & 10 \\
HAR-DVS            & 100 & AdamW  & $0.001$ & Cosine Annealing & 100 & $0.0005$ & 0.1 & 1.0 & 4  \\ \hline \hline
\end{tabular}
\vspace{-3mm}
\end{table*}

\begin{table}[!t]
\caption{Hyperparameter settings of SADP, including the smoothing constant $\beta$ and the maximum pruning ratio $r_{\mathrm{max}}$.}
\label{app:tab:setup_sadp}
\centering
\setlength{\tabcolsep}{20pt}
\begin{tabular}{c|cc}
\hline
\hline
Pruning Ratio (\%)                               & $\beta$ & $r_{\mathrm{max}} (\%)$ \\ \hline
30                   & 0.35  & 60\\ 
50  & 0.3  & 70 \\
70 & 0.2   & 90\\
90      & 0.05 &  100 \\\hline \hline
\end{tabular}
\end{table}

\subsection{Performance Evaluation}
\label{subsec:per_evl}
\subsubsection{Experimental Setups} \textbf{Datasets.} We evaluate SADP on image datasets (CIFAR-10~\cite{cifar}, CIFAR-100~\cite{cifar}, and ImageNet~\cite{deng2009imagenet}) and neuromorphic datasets (CIFAR-10-DVS~\cite{li2017cifar10} and HAR-DVS~\cite{hardvs}) using varying pruning ratios. CIFAR-10~\cite{cifar} and CIFAR-100~\cite{cifar} contain $50,000$ training images and $10,000$ test images, divided into $10$ and $100$ classes, respectively. Standard data augmentation is applied to the training set~\cite{9328869,plif}, including padding each image by $4$ pixels on all sides, followed by a $32\!\times\!32$ crop and random horizontal flipping. Additionally, we employ autoaugment and cutout techniques for further data augmentation~\cite{wang2023adaptive}. ImageNet~\cite{deng2009imagenet} comprises $1,000$ classes, with $1.2$ million images for training and $50,000$ images for validation. Standard data augmentation techniques are used~\cite{sewresnet}. CIFAR10-DVS~\cite{li2017cifar10}, derived from CIFAR-10, is created by scanning each image through repetitive closed-loop movements in front of a DVS camera~\cite{4444573}. It includes $9,000$ training samples and $1,000$ testing samples, each with a spatial resolution of $128\!\times\!128$, which is resized to $48\!\times\!48$. CIFAR10-DVS retains the $10$ classes of CIFAR-10, and no data augmentation techniques are applied to this dataset. HAR-DVS~\cite{hardvs} is an event-based human activity recognition~(HAR) dataset recorded using the DAVIS346 camera at a spatial resolution of $346\times260$. As the \textit{largest} event-based HAR dataset, it encompasses $300$ activity classes with a total of $107,646$ samples. 

\textbf{Network Architectures.} Spiking variants of VGG~\cite{simonyan2014very}, ResNet~\cite{he2016deep, sewresnet}, and Transformer~\cite{yao2024spikedriven} are adopted.

\textbf{Training Settings.} Table~\ref{tab:train_details} summarizes training settings. The hyperparameter setup for SADP is provided in Table~\ref{app:tab:setup_sadp}, and they are selected based on the ablation studies in Section~\ref{subsec:ablation}. We use the same hyperparameters for all datasets at the same pruning ratio. For the experiments on ImageNet with the $35\%$ pruning ratio, the smoothing constant is set to $0.25$, and the maximum pruning ratio is $55\%$. We use the final layer of each network to compute importance scores for computational efficiency. Note that we adopt a two-step training process for ImageNet experiments~\cite{sltt}, which includes a pre-training phase with a single time step for $100$ epochs on SEW-ResNet34~\cite{sewresnet} and $200$ epochs on Meta-SpikeFormer~\cite{yao2024spikedriven}, followed by a fine-tuning phase with $4$ time steps for $10$ epochs on SEW-ResNet34 and $50$ epochs on Meta-SpikeFormer.

\begin{figure}[!t]
\centering
\includegraphics[width=0.92\linewidth]{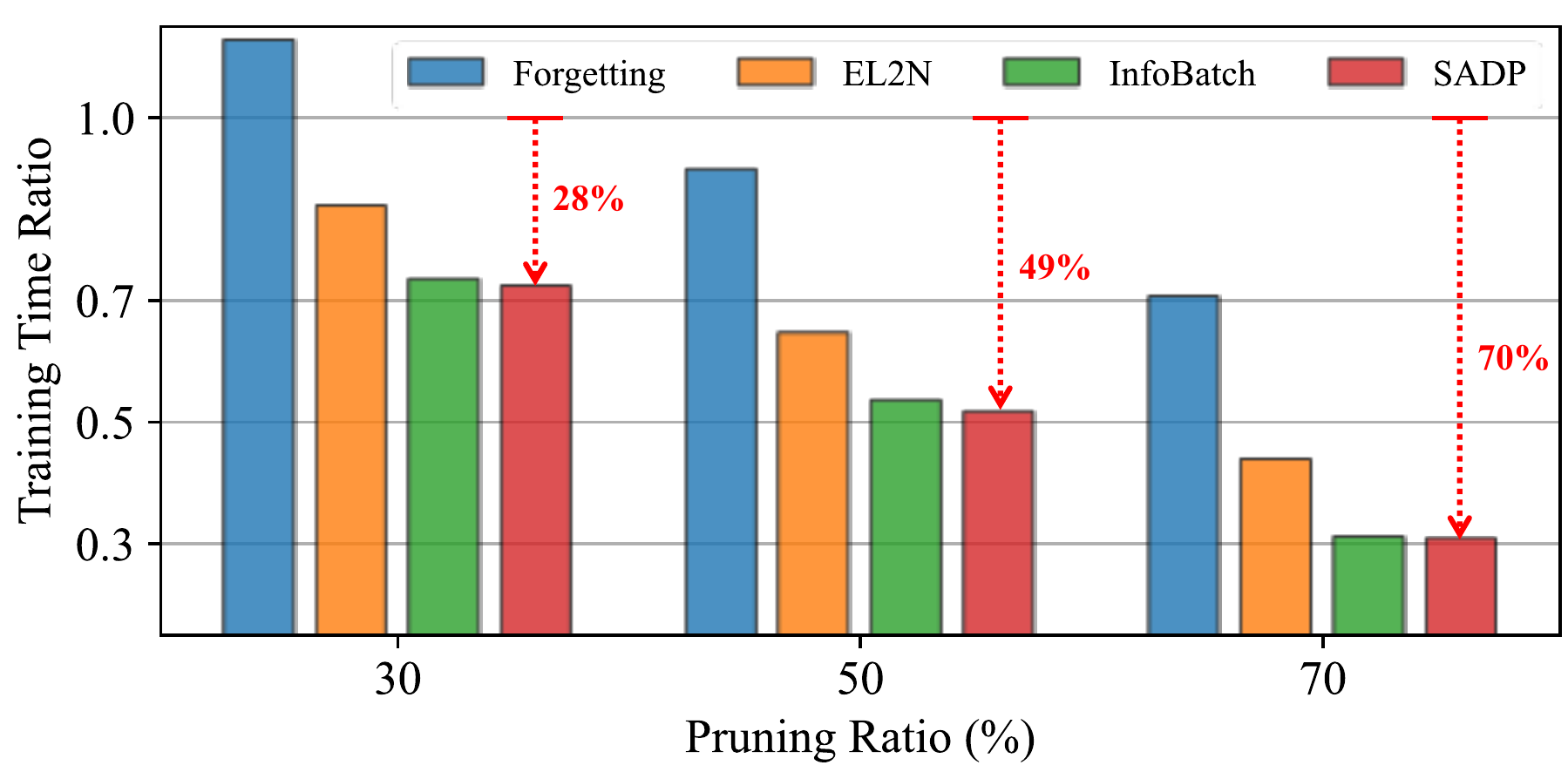}
\vspace{-2mm}
\captionof{figure}{Comparison of training efficiency on CIFAR-100. The y-axis represents the training time ratio relative to the full-data training. The theoretical maximum training time ratio is equal to the specified pruning ratio.}
\label{fig:train_time_ratio}
\end{figure}

\textbf{Baselines.} We compare SADP against existing data pruning methods, including Forgetting~\cite{forgetting}, EL2N~\cite{el2n}, and InfoBatch~\cite{infobatch}, as well as two random baselines (Random Pruning and Reduced Epoch). Random Pruning randomly selects examples at each epoch, while Reduced Epoch uses fewer epochs to match the training iteration number in pruning-based methods. Although these methods were designed for ANNs, they are applicable to SNNs due to their model-agnostic data selection strategies. For fair comparisons, all baseline methods use the same training settings as SADP. 
%In the experiments with InfoBatch~\cite{infobatch}, we carefully adjust the pruning probability hyperparameter to achieve the actual pruning ratio similar to that of other data pruning methods. For larger pruning ratios, such as $70\%$ and $90\%$, we reduce the number of training epochs to align with the expected number of iterations, as recommended in the original paper~\cite{infobatch}.

\subsubsection{SADP Attains State-of-the-Art Accuracy Across Different Pruning Ratios} 
As shown in Table~\ref{tab:baseline}, SADP consistently outperforms all baseline methods across different pruning ratios, datasets, and architectures. Notably, the accuracy advantage of SADP becomes more pronounced at higher pruning ratios, demonstrating its robustness in retaining informative examples under more aggressive data reduction. Additionally, SADP prunes 30\% of the training data on CIFAR-100 and 70\% on CIFAR10-DVS without compromising accuracy. These results highlight the strong generalization of SADP and its effectiveness in preserving model accuracy.

\begin{table*}[!t]
\centering
\caption{Comparison of SADP, existing data pruning methods, and random baselines (Random Pruning and Reduced Epoch) under varying pruning ratios on both image and neuromorphic datasets. Reported values are the mean accuracies and standard deviations over three independent runs.}
\label{tab:baseline}
\resizebox{\textwidth}{!}{%
\setlength{\tabcolsep}{2pt}
\begin{tabular}{l|ccc|ccc|ccc|ccc}
\hline
\hline
Dataset            & \multicolumn{3}{c|}{CIFAR-10}         & \multicolumn{3}{c|}{CIFAR-100}        & \multicolumn{3}{c|}{CIFAR10-DVS}     & \multicolumn{3}{c}{HAR-DVS} \\ \hline
Net (Time Window)            & \multicolumn{3}{c|}{ResNet18 (T=4)}         & \multicolumn{3}{c|}{ResNet19 (T=4)}        & \multicolumn{3}{c|}{VGG11 (T=10)}     & \multicolumn{3}{c}{ResNet18 (T=4)} \\ \hline
Full-data Accuracy          & \multicolumn{3}{c|}{$95.33_{\pm0.04}$}      & \multicolumn{3}{c|}{$80.02_{\pm0.02}$}      & \multicolumn{3}{c|}{$77.83_{\pm0.26}$}      & \multicolumn{3}{c}{$49.03$}   \\ \hline
Pruning Ratio (\%) & 30         & 50         & 70         & 30         & 50         & 70         & 50         & 70         & 90         & 30      & 50      & 70      \\ \hline
Reduced Epoch      & $94.91_{\pm0.08}$ & $94.29_{\pm0.11}$ & $92.93_{\pm0.05}$ & $79.24_{\pm0.07}$ & $78.19_{\pm0.11}$ & $76.19_{\pm0.11}$ & $77.53_{\pm0.29}$ & $76.80_{\pm0.29}$ & $71.37_{\pm0.29}$ & $48.20$   & $46.50$   & $42.23$   \\
Forgetting~\cite{forgetting}         & $95.11_{\pm0.17}$ & $94.47_{\pm0.09}$ & $92.30_{\pm0.19}$ & $78.15_{\pm0.17}$ & $75.07_{\pm0.31}$ & $67.88_{\pm0.30}$ & $70.67_{\pm0.83}$ & $54.63_{\pm2.74}$ & $38.30_{\pm0.94}$ & $46.80$   & $43.78$   & $34.37$   \\
EL2N~\cite{el2n}               & $95.16_{\pm0.04}$ & $94.63_{\pm0.10}$ & $92.07_{\pm0.18}$ & $78.10_{\pm0.12}$ & $73.51_{\pm0.12}$ & $61.48_{\pm0.13}$ & $70.90_{\pm0.41}$ & $61.60_{\pm0.14}$ & $43.97_{\pm0.91}$ & $48.08$   & $46.14$   & $35.49$   \\
Random Pruning     & $94.97_{\pm0.06}$ & $94.25_{\pm0.03}$ & $92.96_{\pm0.15}$ & $79.30_{\pm0.06}$ & $78.24_{\pm0.05}$ & $76.04_{\pm0.08}$ & $77.63_{\pm0.05}$ & $77.03_{\pm0.37}$ & $71.47_{\pm0.12}$ & $48.25$   & $46.58$   & $42.87$   \\
InfoBatch~\cite{infobatch}          & $94.99_{\pm0.16}$ & $94.30_{\pm0.11}$ & $93.02_{\pm0.14}$ & $79.57_{\pm0.12}$ & $78.12_{\pm0.16}$ & $76.24_{\pm0.17}$ & $77.90_{\pm0.16}$ & $77.00_{\pm0.59}$ & $72.30_{\pm0.57}$ & $48.19$   & $46.84$   & $43.27$   \\
SADP &
  $\mathbf{95.22_{\pm0.06}}$ &
  $\mathbf{94.82_{\pm0.05}}$ &
  $\mathbf{93.65_{\pm0.03}}$ &
  $\mathbf{80.01_{\pm0.07}}$ &
  $\mathbf{78.63_{\pm0.09}}$ &
  $\mathbf{77.10_{\pm0.12}}$ &
  $\mathbf{77.97_{\pm0.74}}$ &
  $\mathbf{77.80_{\pm0.70}}$ &
  $\mathbf{74.80_{\pm0.14}}$ &
  $\mathbf{48.51}$ &
  $\mathbf{47.36}$ &
  $\mathbf{44.30}$ \\ \hline
\hline
\end{tabular}%
}
\vspace{-2mm}
\end{table*}

\begin{table*}[!t]
\caption{Comparison of accuracy and wall-clock training time on ImageNet at a $35\%$ pruning ratio. Results are reported with SEW-ResNet34 for 110 epochs and Meta-SpikeFormer-31M for 250 epochs. The wall-clock time (in hours) is calculated as the time per GPU multiplied by the number of GPUs.}
\label{tab:imagenet}
\centering
% \resizebox{0.90\textwidth}{!}{%
\setlength{\tabcolsep}{3pt}
\begin{tabular}{l|cccc|ccc}
\hline
\hline
Net  & \multicolumn{4}{c|}{SEW-ResNet34~\cite{sewresnet} (T=4)}                             & \multicolumn{3}{c}{Meta-SpikeFormer-31M~\cite{yao2024spikedriven} (T=4)}      \\ \hline
Method &
  \multicolumn{1}{c|}{Full Data} &
  Forgetting~\cite{forgetting} &
  EL2N~\cite{el2n} &
  SADP &
  \multicolumn{1}{c|}{Full Data} &
  InfoBatch~\cite{infobatch} &
  SADP \\ \hline
Accuracy & \multicolumn{1}{c|}{68.40} & 65.76 & 66.97  & \textbf{68.29} & \multicolumn{1}{c|}{77.41} & 76.46 & \textbf{76.94} \\
Wall-clock Time (Hours) &
  \multicolumn{1}{c|}{258} &
  209 ($\downarrow$19.0\%) &
  200 ($\downarrow$22.5\%) &
  
  \textbf{168 ($\downarrow$34.9\%)} &
  \multicolumn{1}{c|}{688} &
  456 ($\downarrow$33.7\%) &
  \textbf{448 ($\downarrow$34.9\%)} \\ \hline
\hline
\end{tabular}%
% }
\vspace{-3mm}
\end{table*}

\subsubsection{SADP Achieves the Theoretical Maximum Training Speedup Owing to Negligible Overhead} 
To assess the efficiency of SADP in accelerating SNN training, we measure its wall-clock training time on CIFAR-100 and compare it against baseline methods. Theoretically, a data pruning method with zero overhead can reduce training time by the pruning ratio. However, as shown in Figure~\ref{fig:train_time_ratio}, both Forgetting~\cite{forgetting} and EL2N~\cite{el2n} deviate from this expectation due to the high cost of computing their importance scores. In contrast, SADP consistently achieves training time reductions that closely match the pruning ratios. This indicates that SADP incurs negligible overhead and is highly efficient in practice. Its ability to achieve the theoretical maximum training speedup highlights SADP’s practicality for accelerating SNN training.

\subsubsection{SADP Scales Effectively to Large-scale Datasets}
We evaluate the scalability of SADP on the large-scale ImageNet dataset using two representative architectures: SEW-ResNet34~\cite{sewresnet} and Meta-SpikeFormer~\cite{yao2024spikedriven}. Table~\ref{tab:imagenet} shows that SADP consistently outperforms all baselines in both accuracy and training efficiency. Remarkably, it achieves accuracy comparable to full-data training while reducing training time by 35\% on both architectures. These results demonstrate not only SADP’s effectiveness on large-scale datasets, but also its generalizability across different SNN architectures.

\subsubsection{SADP Significantly Outperforms Baselines at High Pruning Ratios}

% \begin{table}[!htb]
% \caption{Comparison of SADP and baselines at high pruning ratios on CIFAR-10. \textcolor{red}{Table is too small. You may change the x and y-axis}}
% \label{app:tab:high_ratio}
% \centering
% \resizebox{\linewidth}{!}{%
% \begin{tabular}{@{}l|ccccc@{}}
% \hline
% \hline
% Pruning Ratio (\%) & Random Pruning & Forgetting~\cite{forgetting} & EL2N~\cite{el2n}       & InfoBatch~\cite{infobatch}  & SADP                \\ \hline
% 80          & $91.37_{\pm0.12}$     & $87.41_{\pm0.60}$ & $87.51_{\pm0.15}$ & $91.32_{\pm0.10}$ & $\mathbf{92.23_{\pm0.10}}$ \\
% 90          & $87.29_{\pm0.10}$     & $71.13_{\pm0.45}$ & $68.28_{\pm1.41}$ & $87.32_{\pm0.17}$ & $\mathbf{89.04_{\pm0.05}}$ \\ \hline
% \hline
% \end{tabular}%
% }
% \end{table}

\begin{table}[!t]
\caption{Comparison of SADP and baselines at high pruning ratios on CIFAR-10.}
\label{app:tab:high_ratio}
\centering
\setlength{\tabcolsep}{10pt}
\begin{tabular}{l|cc}
\hline
\hline
Method & 80\% & 90\% \\ \hline
Random Pruning & $91.37_{\pm0.12}$ & $87.29_{\pm0.10}$ \\
Forgetting~\cite{forgetting} & $87.41_{\pm0.60}$ & $71.13_{\pm0.45}$ \\
EL2N~\cite{el2n} & $87.51_{\pm0.15}$ & $68.28_{\pm1.41}$ \\
InfoBatch~\cite{infobatch} & $91.32_{\pm0.10}$ & $87.32_{\pm0.17}$ \\
SADP & $\mathbf{92.23_{\pm0.10}}$ & $\mathbf{89.04_{\pm0.05}}$ \\ \hline
\hline
\end{tabular}
\vspace{-2mm}
\end{table}

\begin{table}[!t]
\centering
\caption{Influence of SADP components.}
\label{tab:sadp_components}
\setlength{\tabcolsep}{10pt}
\begin{tabular}{l|cc}
\hline \hline
Dataset                               & CIFAR-10 & CIFAR-100 \\ \hline
Pruning Ratio (\%)                    & 90  & 70\\ \hline
SADP                                 & $89.04_{\pm0.05}$  & $77.10_{\pm0.12}$ \\
SADP w/ Loss Score & $88.37_{\pm0.16}$   & $76.68_{\pm0.13}$\\
SADP w/o Pruning Schedule      & $88.17_{\pm0.13}$ &  $76.49_{\pm0.12}$ \\
SADP w/o Smoothing           & $83.37_{\pm0.27}$  & $72.16_{\pm0.09}$ \\ \hline \hline
\end{tabular}
\vspace{-1mm}
\end{table}

\begin{table}[!t]
\centering
\caption{Influence of the smooth constant $\beta$ and the maximum pruning ratio $r_{\mathrm{max}}$ in SADP on CIFAR-10 at a 50\% pruning ratio.}
\setlength{\tabcolsep}{12pt}
\begin{tabular}{l|c|c|c}
\hline\hline
$\beta$ & Accuracy (\%) & $r_{\mathrm{max}}$ & Accuracy (\%) \\ \hline
0   & $92.36_{\pm0.17}$ & 55\% & $94.62_{\pm0.15}$ \\
0.1 & $94.46_{\pm0.06}$ & 60\% & $94.63_{\pm0.05}$ \\
0.2 & $94.58_{\pm0.03}$ & 65\% & $94.72_{\pm0.05}$ \\
0.3 & $\mathbf{94.82_{\pm0.05}}$ & 70\% & $\mathbf{94.82_{\pm0.05}}$ \\
0.4 & $94.69_{\pm0.10}$ & 75\% & $94.81_{\pm0.08}$ \\
0.5 & $94.47_{\pm0.06}$ & 80\% & $94.75_{\pm0.08}$ \\ \hline\hline
\end{tabular}
\label{tab:smooth_constant}
\end{table}

To evaluate the effectiveness of SADP under more challenging pruning conditions, we compare its performance with existing data pruning methods at high pruning ratios (80\% and 90\%) on CIFAR-10 using ResNet18. As shown in Table~\ref{app:tab:high_ratio}, SADP consistently outperforms all baselines across both pruning levels. For instance, at the 90\% ratio, SADP achieves a 1.7\% higher accuracy than InfoBatch and exceeds EL2N by more than 20\%. These results highlight SADP’s robustness and its superior ability to retain essential informative training examples in excessive pruning regimes where existing methods degrade sharply.

\subsection{Ablation Studies}
\label{subsec:ablation}

\subsubsection{Influence of SADP Components} 
We evaluate the impact of each component in SADP. As shown in Table~\ref{tab:sadp_components}, replacing the spike-aware importance score with loss or removing the dynamic pruning schedule leads to noticeable accuracy drops, especially at higher pruning ratios. Removing the smoothing mechanism causes the most severe degradation, highlighting its role in stabilizing training. These results confirm that all three components are crucial to SADP’s effectiveness.

\subsubsection{Influence of Smoothing Constant}
We investigate the effect of the smoothing constant $\beta$, which balances gradient variance reduction and training stability. Table~\ref{tab:smooth_constant} shows that setting $\beta\!=0$ leads to a sharp drop in accuracy. Increasing $\beta$ improves stability and accuracy, with performance peaking at $\beta\!=0.3$, indicating an optimal trade-off. Notably, this optimal setting generalizes well across datasets and architectures, as evidenced by the superior accuracy in Table~\ref{tab:baseline}.

\subsubsection{Influence of Dynamic Pruning Schedule} 
We analyze the impact of the maximum pruning ratio $r_{\mathrm{max}}$ in the dynamic pruning schedule. Table~\ref{tab:smooth_constant} shows that increasing $r_{\mathrm{max}}$ generally improves accuracy, indicating that preserving more data in the early stages is crucial for better accuracy than later epochs. However, setting $r_{\mathrm{max}}$ too high reduces data availability in the final epochs, leading to diminishing returns. Notably, the optimal setting also generalizes well across datasets and architectures.

\begin{figure}[!t]
\centering
\includegraphics[width=0.89\linewidth]{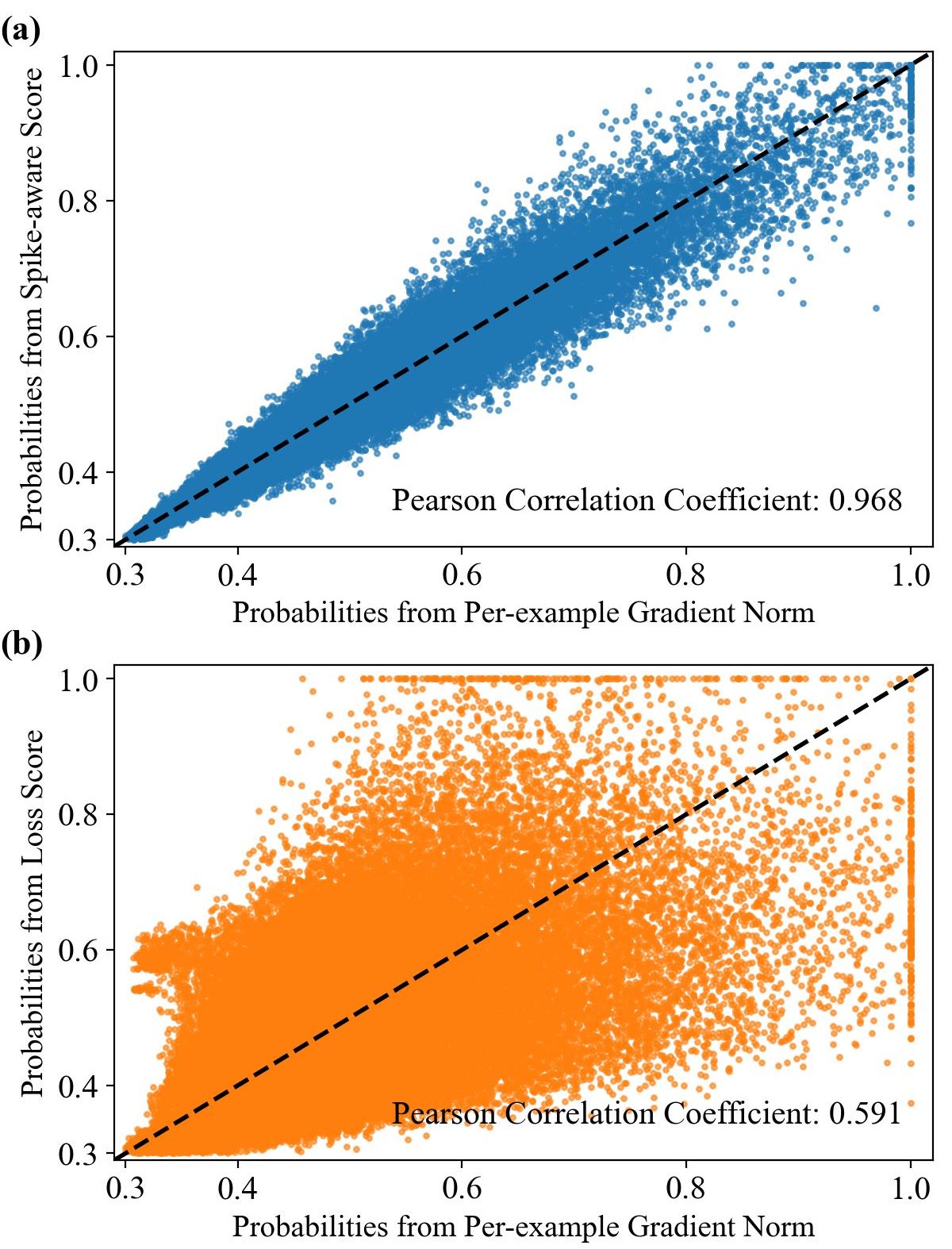}
\vspace{-1.0mm}
\caption{Comparison of selection probabilities derived from (a) the spike-aware importance score and (b) the loss score, against target probabilities based on the per-example gradient norm. The CIFAR10 dataset with a 50\% pruning ratio and a smoothing constant of 0.3 is adopted. Pearson
correlation coefficients are provided in the legend.}
\label{fig:grad_norm_proxy}
\end{figure}

\subsection{Empirical Analysis of Gradient Approximation, Variance Minimization, and Scaling Factors}
\label{subsec:grad_analysis}
\subsubsection{Spike-aware Importance Score Strongly Correlates with the Per-example Gradient Norm}
To assess the efficacy of our spike-aware score as a proxy for the gradient norm, we compare the selection probabilities derived from the spike-aware score, the loss score, and the gradient norm for all examples on CIFAR-10. Figure~\ref{fig:grad_norm_proxy} shows that the spike-aware score exhibits a significantly stronger correlation with the gradient norm compared to the loss score. This confirms that our score more accurately captures each example's training contribution, thereby enabling effective data pruning.

%The primary discrepancies arise with examples having small gradient norms, which should be assigned small probabilities. In these cases, the loss proxy incorrectly assigns larger probabilities, as the loss values do not reflect the true gradient magnitudes. In contrast, our spike-aware proxy effectively accounts for the influence of binary spikes and provides a better approximation of the gradient norm, resulting in more appropriate probabilities for these examples.

\subsubsection{SADP Effectively Reduces Gradient Variance} 
We evaluate the gradient variance of SADP in comparison to Random Pruning and InfoBatch~\cite{infobatch} on CIFAR-10~\cite{cifar} using a 50\% pruning ratio. Figure~\ref{fig:variance} shows that SADP consistently maintains significantly lower gradient variance than both baselines throughout training, and the gap in variance widens over epochs. Additionally, SADP's gradient variance increases with higher values of the smoothing constant, which aligns with our theoretical analysis. These results empirically validate SADP’s ability to minimize gradient variance.

\subsubsection{The Smoothing Strategy Effectively Addresses the Gradient Scaling Issue} To ensure unbiased gradient estimation, SADP rescales the gradient of each selected example by a factor of $\frac{S}{N\cdot p_i}$ according to Eq.~(\ref{eq:est_w_grad}). However, when an example with a very low selection probability $p_i$ is chosen, the resulting scale factor can become excessively large, amplifying the gradient magnitude and destabilizing the training process. To empirically validate this issue, we record the maximum gradient scale factor per epoch on CIFAR-10 under a 50\% pruning ratio, comparing SADP with and without the proposed smoothing mechanism. As shown in Figure~\ref{app:fig:smooth_scale_factor}, SADP without smoothing frequently encounters extreme scaling factors, which leads to unstable updates and lower the accuracy to 92.36\%. In contrast, SADP with smoothing maintains stable scale factors throughout training, resulting in a significantly improved accuracy of 94.82\%. These results underscore the importance of the smoothing strategy in controlling scaling factors for stable training.

\begin{figure}[!t]
\centering
\includegraphics[width=0.85\linewidth]{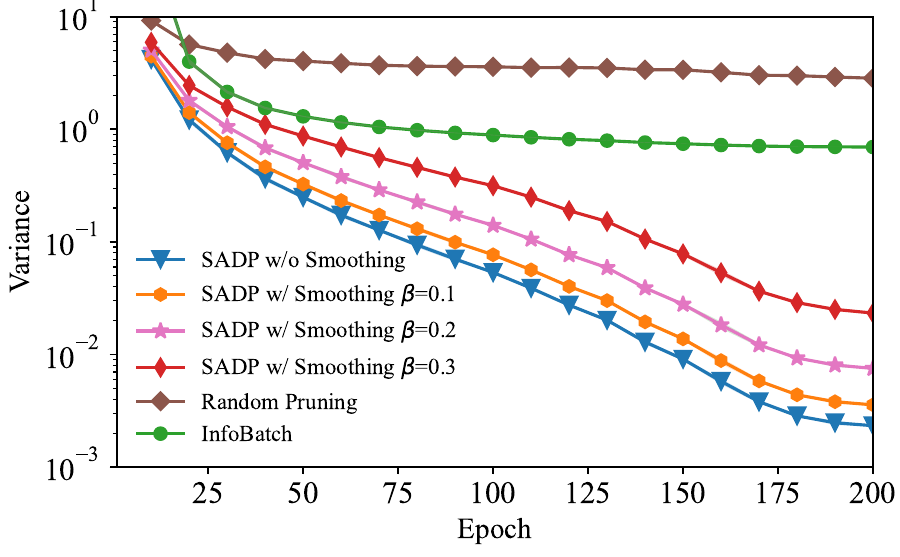}
\vspace{-1.8mm}
\caption{Comparison of gradient variance. Results are shown on CIFAR-10 at a 50\% pruning ratio.}
\label{fig:variance}
\vspace{-1.5mm}
\end{figure}

\begin{figure}[!t]
\centering
\includegraphics[width=0.85\linewidth]{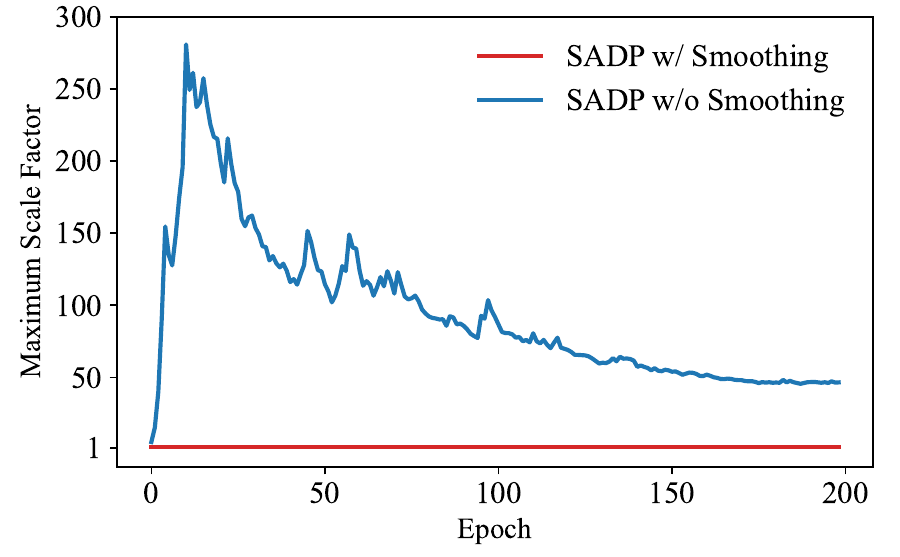}
\vspace{-2mm}
\caption{Comparison of maximum gradient scale factors per epoch with and without the smoothing strategy in SADP. Experiments are conducted on CIFAR-10 at a 50\% pruning ratio.}
\label{app:fig:smooth_scale_factor}
\end{figure}

\subsection{Broad Applicability}
\label{subsec:broad_app}
\subsubsection{SADP Integrates Seamlessly with Efficient SNN Models} To assess the compatibility of SADP with efficient SNN models, we consider two representative examples: the Parallel Spiking Neuron (PSN)~\cite{fang2023parallel}, which facilitates parallel training along the temporal dimension, and the Temporal Reversible SNN (T-RevSNN)~\cite{pmlr-v235-hu24q}, which improves training efficiency through reversible forward propagation. We integrate SADP with both models on ImageNet under a $35\%$ pruning ratio. As summarized in Table~\ref{tab:compatibility}, SADP consistently achieves a $35\%$ reduction in training time without degrading accuracy. These findings highlight that SADP can be combined with advanced SNN models, providing additional training efficiency gains without compromising model performance.

\begin{table}[!t]
\centering		
\caption{Compatibility of SADP with efficient SNN models. We integrate SADP with a 35\% pruning ratio into PSN~\cite{fang2023parallel} and T-RevSNN~\cite{pmlr-v235-hu24q} and compare its accuracy and wall-clock time (in hours) against full-data training on ImageNet for 100 epochs.}
% \vspace{1mm}
% \resizebox{\linewidth}{!}{%
\setlength{\tabcolsep}{1.5pt}
\begin{tabular}{l|c|c|cc}
\hline
\hline
 & Net (Time Window)                  & Method    & Acc. (\%) & Time (Hours)  \\ \hline
\multirow{2}{*}{PSN~\cite{fang2023parallel}}      & \multirow{2}{*}{\makecell[c]{SEW-ResNet18\\(T=4)}}   & Full Data & 67.13     & 104        \\
         &                      & SADP      & 66.94     & 68 ($\downarrow$34.6\%) \\ \hline
\multirow{2}{*}{T-RevSNN~\cite{pmlr-v235-hu24q}} & \multirow{2}{*}{\makecell[c]{ResNet18 (512)\\(T=4)}} & Full Data & 71.54     & 232        \\
         &                      & SADP      & 70.83     & 152 ($\downarrow$34.5\%) \\ \hline \hline
\end{tabular}
% }
\label{tab:compatibility}
\vspace{-2mm}
\end{table}

\begin{table}[!t]
\centering
\caption{Comparison of data pruning methods under online learning.}
\begin{tabular}{l|ccc}
\hline
\hline
Online Learning & \multicolumn{3}{c}{E-Prop~\cite{eprop}} \\
Dataset & \multicolumn{3}{c}{GSC} \\
Net (Time Window) & \multicolumn{3}{c}{Recurrent SNN (T=101)} \\
Full-data Accuracy & \multicolumn{3}{c}{ $89.33\%_{\pm 0.08}$} \\
\hline
Prune Ratio (\%) & 50 & 70 & 90 \\
\hline
Random Sampling & $88.50_{\pm 0.14}$ & $87.68_{\pm 0.12}$ & $83.69_{\pm 0.14}$ \\
InfoBatch~\cite{infobatch} & $88.74_{\pm 0.15}$ & $88.12_{\pm 0.09}$ & $84.60_{\pm 0.20}$ \\
SADP & $\mathbf{89.31_{\pm 0.07}}$ & $\mathbf{88.75_{\pm 0.05}}$ & $\mathbf{85.46_{\pm 0.12}}$ \\
\hline
Online Learning & \multicolumn{3}{c}{SLTT~\cite{sltt}} \\
Dataset & \multicolumn{3}{c}{CIFAR10-DVS} \\
Net (Time Window) & \multicolumn{3}{c}{VGG11 (T=10)} \\
Full-data Accuracy & \multicolumn{3}{c}{ $78.10_{\pm 0.36}$} \\
\hline
Pruning Ratio (\%) & 50 & 70 & 90 \\
\hline
Random Sampling & $77.49_{\pm 0.34}$ & $77.27_{\pm 0.24}$ & $71.83_{\pm 0.73}$ \\
InfoBatch~\cite{infobatch} & $77.50_{\pm 0.45}$ & $77.00_{\pm 0.64}$ & $73.00_{\pm 0.29}$ \\
SADP & $\mathbf{78.10_{\pm 0.70}}$ & $\mathbf{77.97_{\pm 0.45}}$ & $\mathbf{74.97_{\pm 0.17}}$ \\
\hline
\hline
\end{tabular}
\label{app:tab:online_learning}
\end{table}

\subsubsection{SADP is Compatible with Online Learning Rules}
\label{app:sec:online_learning}

In addition to conventional BPTT-based training, online learning algorithms offer an attractive alternative for memory-efficient training of SNNs by continuously updating model parameters at each time step. To assess the effectiveness of SADP under this paradigm, we examine two representative methods: E-Prop~\cite{eprop} and SLTT~\cite{sltt}. SADP is combined with the E-Prop rule on the Google Speech Command (GSC) dataset~\cite{9311226} using a 3-hidden-layer recurrent SNN, and with SLTT on CIFAR10-DVS using ResNet18. The results, summarized in Table~\ref{app:tab:online_learning}, show that SADP consistently surpasses the InfoBatch~\cite{infobatch} and Random Sampling baselines across different pruning ratios. These findings confirm that SADP is fully compatible with online learning rules.
%, underscoring its versatility and practical applicability in diverse SNN training scenarios.

\subsubsection{SADP is Compatible with Local Learning Rules} To further assess its versatility, we evaluate the compatibility of SADP with biologically inspired local learning rules. We integrate SADP with DECOLLE~\cite{decolle} and ELL~\cite{ma2023ell} on CIFAR-10 using ResNet18. Both methods employ layer-wise auxiliary classifiers to locally train each network layer, with DECOLLE using fixed classifiers and ELL adopting trainable ones. As shown in Table~\ref{app:tab:local_learning}, SADP consistently surpasses Random Pruning and InfoBatch~\cite{infobatch} across varying pruning ratios under both schemes, confirming its effectiveness in conjunction with local learning paradigms.

%The generalization of SADP to different training schemes is possible because its selection criterion can be interpreted in terms of the norm of weight updates. The spike-aware score can be redefined under specific learning rules to reflect each sample's contribution to weight changes, enabling SADP to be effective and broadly applicable beyond BPTT.

The strong generalization of SADP  across diverse training schemes can be attributed to the fact that its selection criterion is based on the weight gradient norms. In particular, the spike-aware importance score can be reformulated to quantify each example’s contribution to weight gradient norms, thereby ensuring that SADP remains effective and broadly applicable across different training algorithms.

\subsubsection{SADP is Compatible with Efficient Inference Methods}
Several methods have been proposed to reduce the energy consumption of SNNs during inference, such as quantization-aware training~\cite{wei2025qpsnn} and network pruning~\cite{li2024towards}. While these techniques improve inference efficiency, they often incur substantial training cost in order to preserve model accuracy. To evaluate the compatibility of SADP with such approaches, we integrated SADP into both representative frameworks~\cite{wei2025qpsnn,li2024towards} using their publicly available implementations. As shown in Table~\ref{app:tab:efficient_inference}, SADP consistently outperforms existing data pruning baselines across different pruning ratios under both inference-efficient methods. These results confirm that SADP can be combined with existing efficient inference techniques.

\begin{table}[!t]
\centering
\caption{Comparison of data pruning methods under local learning on the CIFAR-10 dataset using ResNet18 with time window T=4.}
\begin{tabular}{l|ccc}
\hline
\hline
Local Learning & \multicolumn{3}{c}{DECOLLE~\cite{decolle}} \\
Full-data Accuracy & \multicolumn{3}{c}{ $61.96\%_{\pm 0.44}$} \\
\hline
Prune Ratio (\%) & 30   & 50  & 70  \\
\hline
Random Sampling & $60.69_{\pm 0.54}$ & $59.07_{\pm 0.37}$ & $56.45_{\pm 0.33}$ \\
InfoBatch~\cite{infobatch} & $61.39_{\pm 0.41}$ & $59.88_{\pm 0.40}$ & $57.57_{\pm 0.29}$ \\
SADP & $\mathbf{61.84_{\pm 0.22}}$ & $\mathbf{61.22_{\pm 0.45}}$ & $\mathbf{58.54_{\pm 0.33}}$ \\
\hline
Local Learning & \multicolumn{3}{c}{ELL~\cite{ma2023ell}} \\
Full-data Accuracy & \multicolumn{3}{c}{ $87.62\%_{\pm 0.07}$} \\
\hline
Prune Ratio (\%) & 30   & 50  & 70  \\
\hline
Random Sampling & $86.89_{\pm 0.08}$ & $85.87_{\pm 0.11}$ & $84.12_{\pm 0.11}$ \\
InfoBatch~\cite{infobatch} & $87.26_{\pm 0.04}$ & $86.45_{\pm 0.13}$ & $84.51_{\pm 0.17}$ \\
SADP & $\mathbf{87.60_{\pm 0.07}}$ & $\mathbf{87.11_{\pm 0.09}}$ & $\mathbf{85.53_{\pm 0.10}}$ \\
\hline
\hline
\end{tabular}
\label{app:tab:local_learning}
\vspace{-1mm}
\end{table}

\begin{table}[!t]
\centering
\caption{Comparison of data pruning methods under efficient inference on the CIFAR-10 dataset using VGG16 with time window T=4.}
\begin{tabular}{l|ccc}
\hline
\hline
Efficient Inference & \multicolumn{3}{c}{Quantization-aware Training~\cite{wei2025qpsnn} (8 Bits)} \\
Full-data Accuracy & \multicolumn{3}{c}{ $92.04\%_{\pm 0.07}$} \\
\hline
Prune Ratio (\%) & 30 & 50 & 70 \\
\hline
Random Sampling & $90.92_{\pm 0.10}$ & $89.66_{\pm 0.13}$ & $87.64_{\pm 0.07}$ \\
InfoBatch~\cite{infobatch} & $91.16_{\pm 0.04}$ & $90.00_{\pm 0.13}$ & $87.93_{\pm 0.10}$ \\
SADP & $\mathbf{91.65_{\pm 0.04}}$ & $\mathbf{90.68_{\pm 0.11}}$ & $\mathbf{89.05_{\pm 0.09}}$ \\
\hline
Efficient Inference & \multicolumn{3}{c}{Network Pruning~\cite{li2024towards} (Connectivity=30\%)} \\
Full-data Accuracy & \multicolumn{3}{c}{ $90.45\%_{\pm 0.09}$} \\
\hline
Prune Ratio (\%) & 30 & 50 & 70 \\
\hline
Random Sampling & $89.38_{\pm 0.11}$ & $88.52_{\pm 0.03}$ & $86.61_{\pm 0.22}$ \\
InfoBatch~\cite{infobatch} & $89.72_{\pm 0.03}$ & $88.98_{\pm 0.10}$ & $87.26_{\pm 0.07}$ \\
SADP & $\mathbf{90.44_{\pm 0.05}}$ & $\mathbf{89.82_{\pm 0.02}}$ & $\mathbf{88.45_{\pm 0.12}}$ \\
\hline
\hline
\end{tabular}
\label{app:tab:efficient_inference}
\end{table}

\section{Conclusion}
\label{sec:conclusion}
In this article, we proposed SADP, the first effective and efficient data pruning method for SNNs. SADP introduces a spike-aware importance score that identifies informative training examples and assigns selection probabilities proportional to these scores to minimize gradient variance. Both theoretical analyses and empirical results demonstrate that SADP reduces gradient variance, accelerates convergence, and provides an effective approximation of the per-example gradient norm. Extensive experiments further show that SADP consistently outperforms existing data pruning methods and can reduce training time by over 30\% with lossless accuracy across diverse datasets and architectures. Therefore, this work lays a solid foundation for data-efficient training of SNNs and opens new avenues for scalable and efficient training in large-scale neuromorphic systems~\cite{kudithipudi2025neuromorphic}.

% \section*{Acknowledgments}
% This should be a simple paragraph before the References to thank those individuals and institutions who have supported your work on this article.

% {\appendix[Proof of the Zonklar Equations]
% Use $\backslash${\tt{appendix}} if you have a single appendix:
% Do not use $\backslash${\tt{section}} anymore after $\backslash${\tt{appendix}}, only $\backslash${\tt{section*}}.
% If you have multiple appendixes use $\backslash${\tt{appendices}} then use $\backslash${\tt{section}} to start each appendix.
% You must declare a $\backslash${\tt{section}} before using any $\backslash${\tt{subsection}} or using $\backslash${\tt{label}} ($\backslash${\tt{appendices}} by itself
%  starts a section numbered zero.)}


\bibliographystyle{IEEEtran}
\bibliography{myRefs}

% \clearpage
% \onecolumn
% \input{appendix.tex}

% \newpage

 

% \end{IEEEbiography}



\vfill

\end{document}